\documentclass[conference]{IEEEtran}
\usepackage{times}
\usepackage{amsmath}
\usepackage{amsthm}
\usepackage{amssymb}
\usepackage{cancel}
\usepackage{color}
\usepackage{subcaption}
\usepackage{graphicx}
\usepackage{inconsolata}
\usepackage{xcolor}
\usepackage[numbers]{natbib} 
\usepackage[ruled, noline]{algorithm2e}
\usepackage{multicol}
\usepackage{float}
\usepackage{tikz}
\usepackage{ctable}
\usepackage{boldline}
\usepackage[T1]{fontenc}

\usepackage[bookmarks=true,hidelinks]{hyperref}
\usepackage[capitalise]{cleveref} 

\let\originalleft\left
\let\originalright\right
\renewcommand{\left}{\mathopen{}\mathclose\bgroup\originalleft}
\renewcommand{\right}{\aftergroup\egroup\originalright}

\def\argmin{\arg\min}
\def\argmax{\arg\max}

\def\fhat{\hat{f}}
\def\fv{\boldsymbol{f}}
\def\fvhat{\hat{\fv}}

\def\N{\mathcal{N}}
\def\R{\mathbb{R}}
\def\S{\mathbb{S}}
\def\T{^{\mkern-1.5mu\mathsf{T}}}

\def\gv{\boldsymbol{g}}
\def\uv{\boldsymbol{u}}

\def\Fv{\boldsymbol{F}}
\def\Gv{\boldsymbol{G}}
\def\Lv{\boldsymbol{L}}

\def\Wv{\boldsymbol{W}}

\def\Qv{\boldsymbol{Q}}
\def\Rv{\boldsymbol{R}}
\def\rv{\boldsymbol{r}}
\def\xv{\boldsymbol{x}}
\def\wv{\boldsymbol{w}}
\def\uvhat{\hat{\uv}}
\def\wvhat{\hat{\wv}}
\def\xvhat{\hat{\xv}}
\def\xhat{\hat{x}}
\def\uhat{\hat{u}}
\def\what{\hat{w}}

\def\muv{\boldsymbol{\mu}}

\def\thetav{\boldsymbol{\theta}}
\def\Thetav{\boldsymbol{\Theta}}
\def\thetatildev{\tilde{\thetav}}

\def\opt{^\*}

\def\inv{^{-1}}
\def\cat{\mathrm{Cat}}
\def\onehalf{\frac{1}{2}}

\def\diag{\mathrm{diag}}

\def\dregret{\mathrm{D}\text{-}\mathrm{Regret}}
\def\fisher{\mathcal{F}}
\def\const{\mathrm{const.}}
\def\piv{\boldsymbol{\pi}}
\def\th{\ensuremath{^\text{th}}~}

\renewcommand{\d}[1]{\,\mathrm{d} #1}
\newcommand{\E}[1]{\mathbb{E}\left[ #1 \right]}
\newcommand{\Esub}[2]{\mathbb{E}_{#1} \left[ #2 \right]}
\newcommand{\Psub}[2]{\mathbb{P}_{#1} \left( #2 \right)}
\newcommand{\KL}[2]{\mathrm{KL}(#1 \,\|\, #2 )}
\newcommand{\inner}[2]{\langle #1, #2 \rangle}
\newcommand{\bregman}[3]{D_{#1}( #2  \| #3 )}

\newcommand{\norm}[1]{\| #1 \|}
\newcommand{\pol}[1]{\pi_{#1}}
\newcommand{\polv}[1]{\piv_{#1}}
\newcommand{\indicator}[1]{\mathbf{1}\left\{ #1 \right\}}

\newtheorem{fact}{Fact}
\newtheorem{prop}{Proposition}

\usepackage{bm}
\usepackage{bbm}
\usepackage{mathtools}

\def\HH{\mathcal{H}}

\def\MM{\mathcal{M}}\def\NN{\mathcal{N}}

\def\Ebb{\mathbb{E}}

\def\Nbb{\mathbb{N}}
\def\Rbb{\mathbb{R}}

\def\Vbb{\mathbb{V}}

\def\Av{\boldsymbol{A}}
\def\Fv{\boldsymbol{F}}
\def\Gv{\boldsymbol{G}}
\def\Lv{\boldsymbol{L}}

\def\Qv{\boldsymbol{Q}}\def\Rv{\boldsymbol{R}}

\def\Wv{\boldsymbol{W}}

\def\fv{\boldsymbol{f}}
\def\gv{\boldsymbol{g}}

\def\rv{\boldsymbol{r}}
\def\uv{\boldsymbol{u}}
\def\wv{\boldsymbol{w}}\def\xv{\boldsymbol{x}}

\def\R{\Rbb}

\def\const{\mathrm{const.}}
\def\diag{\mathrm{diag}}

\def\*{\star}

\newcommand{\tr}[1]{ \mathrm{tr}\left( #1\right)}

\newcommand\blfootnote[1]{
  \begingroup
  \renewcommand\thefootnote{}\footnote{#1}
  \addtocounter{footnote}{-1}
  \endgroup
}

\def\algfull{Dynamic Mirror Descent Model Predictive Control\xspace}
\def\ouralg{DMD-MPC\xspace}

\begin{document}

\title{An Online Learning Approach to \\ Model Predictive Control}
\author{
	\authorblockN{Nolan Wagener,\authorrefmark{1}\textsuperscript{$\#$}
				  Ching-An Cheng,\authorrefmark{1}\textsuperscript{$\#$}
				  Jacob Sacks,\authorrefmark{2} and
				  Byron Boots\authorrefmark{1}}
	\authorblockA{
				  Georgia Institute of Technology\\
				  \texttt{\{nolan.wagener, cacheng, jsacks\}@gatech.edu}, \texttt{bboots@cc.gatech.edu}}
}

\maketitle

\begin{abstract}
Model predictive control (MPC) is a powerful technique for solving dynamic control tasks.
In this paper, we show that there exists a close connection between MPC and 
online learning, an abstract theoretical framework for analyzing online decision making in the optimization literature.
This new perspective provides a foundation for leveraging powerful online learning algorithms to design MPC algorithms. 
Specifically, we propose a new algorithm based on dynamic mirror descent (DMD), an online learning algorithm that is designed for non-stationary setups. 
Our algorithm, \algfull (\ouralg), represents a general family of MPC algorithms that includes many existing techniques as special instances. 
\ouralg also provides a fresh perspective on previous heuristics used in MPC and suggests a principled way to design new MPC algorithms. 
In the experimental section of this paper, we demonstrate the flexibility of \ouralg, presenting a set of new MPC algorithms on a simple simulated cartpole and a simulated and real-world aggressive driving task. Videos of the real-world experiments can be found at \href{https://youtu.be/vZST3v0_S9w}{\color{blue}\texttt{https://youtu.be/vZST3v0\_S9w}} and \href{https://youtu.be/MhuqiHo2t98}{\color{blue}\texttt{https://youtu.be/MhuqiHo2t98}}.
\end{abstract}

\section{Introduction} \label{sec:intro}
\blfootnote{\kern-1em\authorrefmark{1}Institute for Robotics and Intelligent Machines\\
\authorrefmark{2}School of Electrical and Computer Engineering\\
\textsuperscript{$\#$}Equal contribution}
Model predictive control (MPC)~\cite{MPC-Survey} is an effective tool for control tasks involving dynamic environments, such as helicopter aerobatics \cite{HelicopterAerobatics} and aggressive driving \cite{Williams-Aggressive}.
One reason for its success is the pragmatic principle it adopts in choosing controls:
rather than wasting computational power to optimize a complicated controller for the full-scale problem (which may be difficult to accurately model), 
MPC instead optimizes a simple controller (e.g., an open-loop control sequence) over a shorter planning horizon that is just sufficient to make a sensible decision at the current moment.
By alternating between optimizing the simple controller and applying its corresponding control on the real system, MPC results in a closed-loop policy that can handle modeling errors and dynamic changes in the environment.

Various MPC algorithms have been proposed, using tools ranging from constrained optimization techniques~\cite{camacho2013model, MPC-Survey, Tassa-DDP} to sampling-based techniques~\cite{Williams-Aggressive}.
In this paper, we show that, while these algorithms were originally designed differently
if we view them through the lens of \emph{online learning}~\cite{OCO}, many of them actually follow the {same} general update rule.
Online learning is an abstract theoretical framework for analyzing online decision making. Formally, it concerns iterative interactions between a learner and an environment over $T$ rounds.
At round $t$, the learner makes a decision $\tilde\thetav_t$ from some decision set $\Thetav$. 
The environment then chooses a loss function $\ell_t$ based on the learner's decision, and the learner suffers a cost $\ell_t(\tilde\thetav_t)$.
In addition to  seeing the decision's cost, the learner may be given additional information about the loss function (e.g., its gradient evaluated at $\tilde\thetav_t$) to aid in choosing the next decision $\tilde\thetav_{t+1}$.
The learner's goal is to minimize the accumulated costs $\sum_{t=1}^T \ell_t(\tilde\thetav_t)$, e.g., by minimizing regret~\cite{OCO}.

We find that the MPC process bears a strong similarity with  online learning. 
At time $t$ (i.e., round $t$), an MPC algorithm optimizes a controller (i.e., the decision) over some cost function (i.e., the per-round loss).
To do so, it observes the cost of the initial controller (i.e., $\ell_t(\tilde\thetav_t)$), improves the controller, and executes a control based on the improved controller in the environment to get to the next state (which in turn defines the next per-round loss) with a new controller $\tilde\thetav_{t+1}$.

In view of this connection, we propose a generic framework, \emph{\ouralg} (\algfull), for synthesizing MPC algorithms. 
\ouralg is based on a first-order online learning algorithm called dynamic mirror descent (DMD)~\cite{Hall-DMD}, a generalization of mirror descent~\citep{beck2003mirror} for dynamic comparators. 
We show that several existing MPC algorithms~\citep{Williams-MPPI, Williams-IT-MPC} are special cases of \ouralg, given specific choices of step sizes, loss functions, and regularization.
Furthermore, we demonstrate how new MPC algorithms can be derived systematically from \ouralg with only mild assumptions on the regularity of the cost function.
This allows us to even work with discontinuous cost functions (like indicators) and discrete controls.
Thus, \ouralg offers a spectrum from which practitioners can easily customize new algorithms for their applications. 

In the experiments, we apply \ouralg to design a range of MPC algorithms and study their empirical performance.
Our results indicate the extra design flexibility offered by \ouralg does make a difference in practice;
by properly selecting hyperparameters which are obscured in the previous approaches, we are able to improve the performance of existing algorithms.
Finally, we apply \ouralg on a real-world AutoRally car platform~\citep{AutoRally} for autonomous driving tasks and show it can achieve competent performance.

\vspace{1em}
\textit{Notation:}
As our discussions will involve planning horizons, for clarity, we use lightface to denote variables that are meant for a single time step, and boldface to denote the variables congregated across the MPC planning horizon.
For example, we use $\uhat_t$ to denote the planned control at time $t$ and $\uvhat_t \triangleq (\uhat_t, \ldots, \uhat_{t+H-1})$ to denote an $H$-step planned control sequence starting from time $t$.
We use a subscript to extract elements from a congregated variable; e.g., we use $\uhat_{t,h}$ to the denote the $h$\th element in $\uvhat_t$ (the subscript index starts from zero). 
All the variables in this paper are finite-dimensional.

\section{An Online Learning Perspective on MPC} \label{sec:mpc_ol}

\subsection{The MPC Problem Setup} \label{sec:mpc setup}
Let $n,~m \in \Nbb_+$ be finite. We consider the problem of controlling a discrete-time stochastic dynamical system
\begin{align} \label{eq:true dynamics}
x_{t+1} \sim f(x_t, u_t)
\end{align}
for some stochastic transition map $f : \R^n \times \R^m  \to \R^n$. At time $t$, the system is in state $x_t \in \R^n$.
Upon the execution of control $u_t \in \R^m$, the system randomly transitions to the next state $x_{t+1}$, and an instantaneous cost $c(x_t, u_t)$ is incurred.
Our goal is to design a state-feedback control law (i.e., a rule of choosing $u_t$ based on $x_t$) such that the system exhibits good performance (e.g., accumulating low costs over $T$ time steps).

In this paper, we adopt the MPC approach to choosing $u_t$:
at state $x_t$, we imagine controlling a stochastic dynamics model $\fhat$ (which approximates our system $f$) for $H$ time steps into the future. 
Our planned controls come from a control distribution $\polv{\thetav}$ that is parameterized by some vector $\thetav \in \Thetav$, where $\Thetav$ is the feasible parameter set.
In each simulation (i.e., rollout), we sample\footnote{This can be sampled in either an open-loop or closed-loop fashion.}
a control sequence $\uvhat_t$ from the control distribution $\polv{\thetav}$ and recursively apply it to $\fhat$ to generate a predicted state trajectory $\xvhat_t \triangleq (\xhat_t, \xhat_{t+1}, \ldots, \xhat_{t+H})$: let $\xhat_t = x_t $; for \mbox{$\tau = t,\dots, t+H-1$}, we set $\xhat_{\tau+1} \sim \fhat(\xhat_\tau, \uhat_\tau)$.
More compactly, we can write the simulation process as
\begin{align} \label{eq:congregated model dynamics}
\xvhat_t  \sim \fvhat(x_t, \uvhat_t)
\end{align}
in terms of some $\fvhat$ that is defined naturally according to the above recursion. 
Through these simulations, we desire to select a parameter $\thetav_t \in \Thetav$ that minimizes an MPC objective $\hat{J}(\polv{\thetav}; x_t)$, which aims to predict the performance of the system if we were to apply the control distribution $\polv{\thetav}$ starting from $x_t$.\footnote{$\hat{J}$ can be seen as a surrogate for the long-term performance of our controller.
Typically, we set the planning horizon $H$ to be much smaller than $T$ to reduce the optimization difficulty and to mitigate modeling errors.}
In other words,  we wish to find the $\thetav_t$ that solves 
\begin{equation} \label{eq:mpc_obj}
\min_{\thetav \in \Thetav}\hat{J}(\polv{\thetav}; x_t).
\end{equation}
Once $\thetav_t$ is decided, we then sample\footnote{This setup can also optimize deterministic policies, e.g., by defining $\polv{\thetav}$ to be a Gaussian policy with the mean being the deterministic policy.}
$\uvhat_t$ from $\polv{\thetav_t}$, extract the first control $\uhat_t$, and apply it on the real dynamical system $f$ in~\eqref{eq:true dynamics} (i.e., set $u_t = \uhat_t$) to go to the next state $x_{t+1}$. Because $\thetav_t$ is determined based on $x_t$,  MPC is effectively state-feedback.
  
The motivation behind MPC is to use the MPC objective $\hat{J}$ to reason about the controls required to achieve desirable long-term behaviors.
Consider the statistic
\begin{align} \label{eq:sum of costs}
C(\xvhat_t, \uvhat_t) \triangleq \sum_{h=0}^{H-1} c(\xhat_{t+h}, \uhat_{t+h}) + c_{\mathrm{end}}(\xhat_{t+H}),
\end{align}
where $c_{\mathrm{end}}$ is a terminal cost function. 
A popular MPC objective is $\hat{J}(\polv{\thetav};x_t) = \mathbb{E}[C(\xvhat_t, \uvhat_t) \mid x_t, \polv{\thetav}, \fvhat]$, which estimates the expected $H$-step future costs. 
Later in~\cref{sec:opt}, we will discuss several MPC objectives and their properties.

Although the idea of MPC sounds intuitively promising, the optimization can only be approximated in practice (e.g., using an iterative algorithm like gradient descent), because~\eqref{eq:mpc_obj} is often a stochastic program (like the example above) and the control command $u_t$ needs to be computed at a high frequency.
In consideration of this imperfection, it is common to heuristically \emph{bootstrap} the previous approximate solution as the initialization to the current problem. 
Specifically, let $\thetav_{t-1}$ be the approximate solution to the previous problem and $\tilde{\thetav}_{t}$ denote the initial condition of $\thetav$ in solving~\eqref{eq:mpc_obj}. The bootstrapping step can then written as 
\begin{align} \label{eq:shift operation}
\tilde{\thetav}_{t} = \Phi(\thetav_{t-1})
\end{align}
by effectively defining a \emph{shift operator} $\Phi$ (see \cref{app:shift operator} for details).
Because the subproblems in~\eqref{eq:mpc_obj} of two consecutive time steps share all control variables except for the first and the last ones, shifting the previous solution provides a warm start to~\eqref{eq:mpc_obj} to amortize the computational complexity.

\subsection{The Online Learning Perspective} \label{sec:online learning setup}

\begin{figure}
	\centering
	\begin{tikzpicture}[scale=0.55, every node/.style={scale=0.8}]
	\node() at (-7, 0) {};
	\node (xtm2) at (-6.5, 0) {};
	\node (xt) at (0, 0) {$x_t$};
	\node (xtm1) at (-4, 0) {$x_{t-1}$};
	\node (xtp1) at (4, 0) {$x_{t+1}$};
	\node (xtp2) at (6.7, 0) {};
	\node (utm2) at (-5, 1) {$u_{t-2}$};
	\node (utm1) at (-1, 1) {$u_{t-1}$};
	\node (ut) at (3, 1) {$u_t$};
	\node (utp1) at (6.7, 1) {};
	
	\node (thtm2) at (-6.5, 3) {};
	\node (tilthtm1) at (-4.3, 3) {$\tilde\thetav_{t-1}$};
	\node (thtm1) at (-2.6, 3) {$\thetav_{t-1}$};
	\node (tiltht) at (-0.3, 3) {$\tilde\thetav_t$};
	\node (tht) at (1.5, 3) {$\thetav_t$};
	\node (tilthtp1) at (3.7, 3) {$\tilde\thetav_{t+1}$};
	\node (thtp1) at (5.4, 3) {$\thetav_{t+1}$};
	\node (tilthtp2) at (6.7, 3) {};
	\node (elltm1) at (-4, 2) {$\ell_{t-1}$};
	\node (ellt) at (0, 2) {$\ell_t$};
	\node (elltp1) at (4, 2) {$\ell_{t+1}$};
	
	\node at (-4, 4) {round $t-1$};
	\node at (0, 4) {round $t$};
	\node at (4, 4) {round $t+1$};
	
	\draw[->, blue, thick] (thtm2.south west) -- (utm2.north west);
	\draw[->, blue, thick] (thtm1.south east)  -- (utm1.north west);
	\draw[->, blue, thick] (tht.south east) -- (ut.north west);
	\draw[->, blue, thick] (thtp1.south east) -- (utp1.north west);
	\draw[->, blue, thick] (thtm2.east) -- (tilthtm1.west);
	\draw[->, blue, thick] (thtm1.east) -- (tiltht.west) node [midway, above] {{\color{black}~{$\Phi$}}};
	\draw[->, blue, thick] (tht.east) -- (tilthtp1.west) node [midway, above] {{\color{black}~{$\Phi$}}};
	\draw[->, blue, thick] (thtp1.east) -- (tilthtp2.west);
	\draw[->, red, thick] (xtm1.north) -- (elltm1.south);
	\draw[->, red, thick] (xt.north) -- (ellt.south);
	\draw[->, red, thick] (xtp1.north) -- (elltp1.south);
	
	\draw[->, red, thick] (xtm2.east) -- (xtm1.west);
	\draw[->, red, thick] (xtm1.east) -- (xt.west);
	\draw[->, red, thick] (xt.east) -- (xtp1.west);
	\draw[->, red, thick] (xtp1.east) -- (xtp2.west);
	\draw[->, red, thick] (utm2.south) -- (xtm1.north west);
	\draw[->, red, thick] (utm1.south) -- (xt.north west);
	\draw[->, red, thick] (ut.south) -- (xtp1.north west);
	
	\draw[->, blue, thick] (tilthtm1.east) -- (thtm1.west);
	\draw[->, blue, thick] (elltm1.east) -- (thtm1.south west);
	\draw[->, blue, thick] (tiltht.east) -- (tht.west);
	\draw[->, blue, thick] (ellt.east) -- (tht.south west);
	\draw[->, blue, thick] (tilthtp1.east) -- (thtp1.west);
	\draw[->, blue, thick] (elltp1.east) -- (thtp1.south west);
	
	\draw[densely dotted, thick] (-6, -0.5) -- (-6, 4.25);
	\draw[densely dotted, thick] (-2, -0.5) -- (-2, 4.25);
	\draw[densely dotted, thick] (2, -0.5) -- (2, 4.25);
	\draw[densely dotted, thick] (6, -0.5) -- (6, 4.25);
	\end{tikzpicture}
	\caption{Diagram of the online learning perspective, where blue and red denote the learner and the environment, respectively.} \label{fig:online learning perspective}
\end{figure}
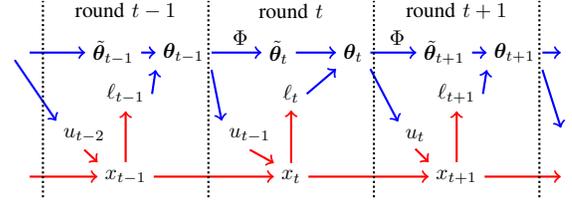

As discussed, the iterative update process of MPC resembles the setup of online learning~\citep{OCO}. Here we provide the details to convert an MPC setup into an online learning problem. Recall from the introduction that online learning mainly consists of three components: the decision set, the learner's strategy for updating decisions, and the environment's strategy for updating per-round losses. We show the counterparts in MPC that correspond to each component below. Note that in this section we will overload the notation $\hat{J}(\thetav; x_t)$ to mean $\hat{J}(\polv{\thetav}; x_t)$.

We use the concept of per-round loss in online learning as a mechanism to measure the decision uncertainty in MPC, and propose the following identification (shown in~\cref{fig:online learning perspective}) for the MPC setup described in the previous section: 
we set the rounds in online learning to synchronize with the time steps of our control system, set the decision set $\Thetav$ as the space of feasible parameters of the control distribution $\polv{\thetav}$, set the learner as the MPC algorithm which in round $t$ outputs the decision $\tilde{\thetav}_t \in \Thetav$ and side information $u_{t-1}$, and set the per-round loss as
\begin{align} \label{eq:per-round loss}
\ell_t(\cdot) = \hat{J}(\cdot\, ; x_t).
\end{align}
In other words, in round $t$ of this online learning setup, the learner plays a decision $\tilde{\thetav}_t$ along with a side information $u_{t-1}$ (based on the optimized solution $\thetav_{t-1}$ and the shift operator in~\eqref{eq:shift operation}), the environment selects  the per-round loss \mbox{$\ell_t(\cdot) = \hat{J}(\cdot; x_t)$} (by applying $u_{t-1}$ to the real dynamical system in~\eqref{eq:true dynamics} to transit the state to $x_t$), and finally the learner receives $\ell_t$ and incurs cost $\ell_t(\tilde\thetav_t)$ (which measures the sub-optimality of the future plan made by the MPC algorithm).

This online learning setup differs slightly from the standard setup in its separation of the decision $\tilde{\thetav}_t$ and the side information $u_{t-1}$; while our setup can be converted into a standard one that treats $\thetav_{t-1}$ as the sole decision played in round $t$, we adopt this explicit separation in order to emphasize that the variable part of the incurred cost $\ell_t(\tilde\thetav_t)$  pertains to only $\tilde\thetav_t$. 
That is, the learner cannot go back and revert the previous control $u_{t-1}$ already applied on the system, but only uses $\ell_t$ to update the current and future controls $\uhat_t, \dots, \uhat_{t+H-1}$.

The performance of the learner in online learning (which by our identification is the MPC algorithm) is measured in terms of the accumulated costs $\sum_{t=1}^{T} \ell_t(\tilde{\thetav}_t)$. 
For problems in non-stationary setups, a normalized way to describe the accumulated costs in the online learning literature
is through the concept of \emph{dynamic regret}~\cite{Hall-DMD, AdaptiveOL}, which is defined as   
\begin{align} \label{eq:dynamic regret}
\dregret = \sum_{t=1}^T \ell_t(\tilde{\thetav}_t) - \sum_{t=1}^T  \ell_t(\thetav_t\opt),
\end{align}
where $\thetav_t\opt \in \argmin_{\thetav \in \Thetav}\, \ell_t(\thetav)$. 
Dynamic regret quantifies how suboptimal the played decisions $\tilde{\thetav}_1,\dots,\tilde{\thetav}_T$ are on the corresponding loss functions. 
In our proposed problem setup, the optimality concept associated with dynamic regret conveys a \emph{consistency} criterion desirable for MPC: we would like to make a decision $\thetav_{t-1}$ at state $x_{t-1}$
such that, after applying control $u_{t-1}$ and entering the new state $x_t$, its shifted plan $\tilde\thetav_t$ remains close to optimal with respect to the new loss function $\ell_t$.
If the dynamics model $\fhat$ is accurate and the MPC algorithm is ideally solving~\eqref{eq:mpc_obj}, we can expect that bootstrapping the previous solution $\thetav_{t-1}$ through~\eqref{eq:shift operation} into $\thetatildev_t$  would result in a small instantaneous gap $\ell_t(\tilde{\thetav}_t) - \ell_t(\thetav_t\opt)$ which is solely due to unpredictable future information (such as the stochasticity in the dynamical system).
In other words, an online learning algorithm with small dynamic regret, if applied to our online learning setup, would produce a consistently optimal MPC algorithm with regard to the solution concept discussed above. However, we note that having small dynamic regret here does not directly imply good absolute performance on the control system, because the overall performance of the MPC algorithm is largely dependent on the form of the MPC objective $\hat{J}$ (e.g., through choice of $H$ and accuracy of $\fhat$).
Small dynamic regret more precisely means whether the plan produced by an MPC algorithm is consistent with the given MPC objective.

\section{A Family of MPC Algorithms Based on Dynamic Mirror Descent}

The online learning perspective on MPC suggests that good MPC algorithms can be designed from online learning algorithms that achieve small dynamic regret. This is indeed the case. We will show that a range of existing MPC algorithms are in essence applications of a classical online learning algorithm called dynamic mirror descent (DMD)~\citep{Hall-DMD}.
DMD is a generalization of mirror descent~\citep{beck2003mirror} to problems involving dynamic comparators (in this case, the $\{\thetav_t^\*\}$ in dynamic regret in~\eqref{eq:dynamic regret}). In round $t$, DMD applies the following update rule:
\begin{align}\label{eq:DMD}
\begin{split}
\thetav_t &= \argmin_{\thetav \in \Thetav}\,  \inner{\gamma_t \gv_t}{\thetav} + \bregman{\psi}{\thetav}{\tilde\thetav_t} \\
\tilde\thetav_{t+1} &= \Phi(\thetav_{t}),
\end{split}
\end{align}
where $\gv_t = \nabla \ell_t(\tilde\thetav_t)$
 (which can be replaced by unbiased sampling if $\nabla \ell_t(\tilde\thetav_t)$ is an expectation), 
$\Phi$ is called the \emph{shift model},\footnote{In~\citep{Hall-DMD}, $\Phi$ is called a \emph{dynamical model}, but it is not the same as the dynamics of our control system.
We therefore rename it to avoid confusion.} $\gamma_t > 0$ is the step size, and for some $\thetav, \thetav' \in \Thetav$,
$
\bregman{\psi}{\thetav}{\thetav'} \triangleq \psi(\thetav) - \psi(\thetav') - \inner{\nabla \psi(\thetav')}{\thetav - \thetav'}
$
is the Bregman divergence generated by a strictly convex function $\psi$ on $\Thetav$. 

The first step of DMD in~\eqref{eq:DMD} is reminiscent of the proximal update in the usual mirror descent algorithm. It can be thought of as an optimization step where the Bregman divergence acts as a regularization to keep $\thetav$ close to $\tilde\thetav_t$.
Although $\bregman{\psi}{\thetav}{\thetav'}$ is not necessarily a metric (since it may not be symmetric), it is still useful to view it as a distance between $\thetav$ and $\thetav'$.
Indeed, familiar examples of the Bregman divergence include the squared Euclidean distance and KL divergence\footnote{For probability distributions $p$ and $q$ over a random variable $x$, the KL divergence is defined as $\KL{p}{q} \triangleq \Ebb_{x \sim p}\left[ \log \frac{p(x)}{q(x)} \right]$.}~\citep{banerjee2005clustering}.

The second step of DMD in~\eqref{eq:DMD} uses the shift model $\Phi$ to anticipate the optimal decision for the next round. In the context of MPC, a natural choice for the shift model is the shift operator in~\eqref{eq:shift operation} defined previously in~\cref{sec:mpc setup} (hence the same notation), because the per-round losses in two consecutive rounds here concern problems with shifted time indices.
\citet{Hall-DMD} show that the dynamic regret of DMD scales with how much the optimal decision sequence $\{\thetav_t^\*\}$ deviates from $\Phi$ (i.e., $\sum_t \norm{\thetav_{t+1}^\* - \Phi(\thetav_t^\*)})$, which is proportional to the unpredictable elements of the problem.

\begin{algorithm}
	\caption{{\small Dynamic Mirror Descent MPC (\ouralg)}}
	\label{alg:dmd-mpc}
	
	\For{$t = 1, 2, \ldots, T$} {
		$\ell_t(\cdot) = \hat{J}(\cdot\,; x_t)$
		
		$\thetav_t = \displaystyle\argmin_{\thetav \in \Thetav} \; \inner{\gamma_t\nabla \ell_t(\thetatildev_t)}{\thetav} + \bregman{\psi}{\thetav}{\thetatildev_t}$
				
		Sample $\uvhat_t \sim \polv{\thetav_t}$ and set $u_t = \uhat_t$
		
		Sample $x_{t+1} \sim f(x_t, u_t)$
		
		$\thetatildev_{t+1} = \Phi(\thetav_t)$ 
	}
\end{algorithm}

\noindent Applying DMD in~\eqref{eq:DMD} to the online learning problem described in~\cref{sec:online learning setup} leads to an MPC algorithm shown in~\cref{alg:dmd-mpc}, which we call \emph{\ouralg}.
More precisely, \ouralg represents a family of MPC algorithms in which a specific instance is defined by a choice of:
\begin{enumerate}
\item the MPC objective $\hat{J}$ in~\eqref{eq:per-round loss},
\item the form of the control distribution $\polv{\thetav}$, and
\item  the Bregman divergence $D_\psi$ in~\eqref{eq:DMD}. 
\end{enumerate}
Thus, we can use \ouralg as a generic strategy for synthesizing MPC algorithms.
In the following, we use this recipe to recreate several existing MPC algorithms and demonstrate new MPC algorithms that naturally arise from this framework.

\subsection{Loss Functions} \label{sec:opt}

We discuss several definitions of the per-round loss $\ell_t$, which all result from the formulation in~\eqref{eq:per-round loss} but with different  $\hat{J}$. These loss functions are based on the statistic $C(\hat\xv_t, \hat\uv_t)$ defined in~\eqref{eq:sum of costs} which measures the $H$-step accumulated cost of a given trajectory.
For transparency of exposition, we will suppose henceforth that the control distribution $\polv{\thetav}$ is open-loop\footnote{Note again that even while using open-loop control distributions, the overall control law of MPC is state-feedback.}; similar derivations follow naturally for closed-loop control distributions. 
For convenience of practitioners, we also provide expressions of their gradients in terms of the likelihood-ratio derivative\footnote{We assume the control distribution is sufficiently regular with respect to its parameter so that the likelihood-ratio derivative rule holds.}~\citep{glynn1990likelihood}. For some function $L_t(\hat\xv_t, \hat\uv_t)$, all these gradients shall have the form
\begin{small}
\begin{equation}\label{eq:likelihood-ratio}
\nabla \ell_t(\thetav) = \Ebb_{\hat\uv_t \sim \polv{\thetav}} \Esub{\hat\xv_t \sim \hat\fv(x_t, \hat\uv_t)}{L_t(\hat\xv_t, \hat\uv_t) \nabla_{\thetav} \log \polv{\thetav}(\hat\uv_t)}.
\end{equation}
\end{small}
In short, we will denote $\Ebb_{\hat\uv_t \sim \polv{\thetav}} \Ebb_{\hat\xv_t \sim \hat\fv(x_t, \hat\uv_t)}$ as $\Ebb_{\polv{\thetav}, \hat\fv}$. These gradients in practice are approximated by finite samples.

\subsubsection{Expected Cost}
The most commonly used MPC objective is the $H$-step expected accumulated cost function under model dynamics,
because it directly estimates the expected long-term behavior when the dynamics model $\fhat$ is accurate and $H$ is large enough. Its per-round loss function is\footnote{In experiments, we subtract the empirical average of the sampled costs from $C$ in~\eqref{eq:grad expected cost (MPC obj)} to reduce the variance, at the cost of a small amount of bias.}
\begin{align}  \label{eq:expected cost (MPC obj)}
\ell_t(\thetav) &= \Esub{\polv{\thetav}, \hat\fv}{C(\hat\xv_t, \hat\uv_t)} \\[1ex]
\nabla \ell_t(\thetav) &= \Esub{\polv{\thetav}, \hat\fv}{C(\hat\xv_t, \hat\uv_t) \nabla_{\thetav} \log \polv{\thetav}(\hat\uv_t)}. \label{eq:grad expected cost (MPC obj)}
\end{align}

\subsubsection{Expected Utility}
Instead of optimizing for average cost, we may care to optimize for some preference related to the trajectory cost $C$, such as having the cost be below some threshold.
This idea can be formulated as a \emph{utility} that returns a normalized score related to the preference for a given trajectory cost $C(\hat\xv_t, \hat\uv_t)$.
Specifically, suppose that $C$ is lower bounded by zero\footnote{If this is not the case, let $c_{\min} \triangleq \inf_{\hat\xv_t, \hat\uv_t} C(\hat\xv_t, \hat\uv_t)$, which we assume is finite.
We can then replace $C$ with $\tilde C(\hat\xv_t, \hat\uv_t) \triangleq C(\hat\xv_t, \hat\uv_t) - c_{\min}$.} and at some round $t$ define the utility $U_t : \R_+ \to [0, 1]$ (i.e., $U_t :  C(\hat\xv_t, \hat\uv_t) \mapsto U_t( C(\hat\xv_t, \hat\uv_t))$) to be a function with the following properties:
$U_t(0) = 1$,
\mbox{$U_t$ is monotonically decreasing,} and
$\lim_{z \to +\infty} U_t(z) = 0$.
These are sensible properties since we attain maximum utility when we have zero cost, the utility never increases with the cost, and the utility approaches zero as the cost increases without bound.
We then define the per-round loss as
\begin{align} \label{eq:expected utility (MPC obj)}
\ell_t(\thetav) &= -\log \Esub{\polv{\thetav}, \hat\fv}{U_t(C(\hat\xv_t, \hat\uv_t))} \\[1ex]
\nabla \ell_t(\thetav) &= -\frac{\Esub{\polv{\thetav}, \hat\fv}{U_t(C(\hat\xv_t, \hat\uv_t)) \nabla_{\thetav} \log \polv{\thetav}(\hat\uv_t)}}{\Esub{\polv{\thetav}, \hat\fv}{U_t(C(\hat\xv_t, \hat\uv_t))}}. \label{eq:grad expected utility (MPC obj)}
\end{align}
The gradient in~\eqref{eq:grad expected utility (MPC obj)} is particularly appealing when estimated with samples.
Suppose we sample $N$ control sequences $\hat\uv_t^1, \dots, \hat\uv_t^N$ from $\polv{\thetav}$ and (for the sake of compactness) sample one state trajectory from $\hat\fv$ for each corresponding control sequence, resulting in $\hat\xv_t^1, \dots, \hat\xv_t^N$. Then the estimate of~\eqref{eq:grad expected utility (MPC obj)} is a convex combination of gradients:
\begin{align*}
\nabla \ell_t(\thetav) \approx -\sum_{i=1}^N w_i \nabla_{\thetav} \log \polv{\thetav}(\hat\uv_t^i),
\end{align*}
where $w_i = \frac{U_t(C_i)}{\sum_{j=1}^N U_t(C_j)} $ and $ C_i = C(\hat\xv_t^i, \hat\uv_t^i)$, for \mbox{$i = 1,\dots,N$}.
We see that each weight $w_i$ is computed by considering the relative utility of its corresponding trajectory. A cost $C_i$ with high relative utility will push its corresponding weight $w_i$ closer to one, whereas a low relative utility will cause $w_i$ to be close to zero, effectively rejecting the corresponding sample.

We give two examples of utilities and their related losses.

\paragraph{Probability of Low Cost}
For example, we may care about the system being below some cost threshold as often as possible.
To encode this preference, we can use the threshold utility $U_t(C) \triangleq \indicator{C \le C_{t,\max}}$, where $\indicator{\cdot}$ is the indicator function and $C_{t,\max}$ is a threshold parameter.
Under this choice, the loss and its gradient become
\begin{small}
\begin{align} \label{eq:low cost probability (MPC obj)}
\ell_t(\thetav) &= -\log \Esub{\polv{\thetav}, \hat\fv}{\indicator{C(\hat\xv_t, \hat\uv_t) \le C_{t,\max}}} \\
				&= -\log \Psub{\polv{\thetav}, \hat\fv}{C(\hat\xv_t, \hat\uv_t) \le C_{t,\max}} \nonumber \\
\nabla \ell_t(\thetav) &= -\frac{\Esub{\polv{\thetav}, \hat\fv}{\indicator{C(\hat\xv_t, \hat\uv_t) \le C_{t,\max}} \nabla_{\thetav} \log \polv{\thetav}(\hat\uv_t)}}
								{\Esub{\polv{\thetav}, \hat\fv}{\indicator{C(\hat\xv_t, \hat\uv_t) \le C_{t,\max}}}}.
\end{align}
\end{small}
As we can see, this loss function also gives the probability of achieving cost below some threshold.
As a result~(\cref{fig:indicator}), costs below $C_{t,\max}$ are treated the same in terms of the utility.
This can potentially make optimization easier since we are trying to make good trajectories as likely as possible instead of finding the best trajectories as in~\eqref{eq:expected cost (MPC obj)}.

However, if the threshold $C_{t,\max}$ is set too low and the gradient is estimated with samples, the gradient estimate may have high variance due to the large number of rejected samples.
Because of this, in practice, the threshold is set adaptively, e.g., as the largest cost of the top \emph{elite fraction} of the sampled trajectories with smallest costs~\citep{CEM}.
This allows the controller to make the best sampled trajectories more likely and therefore improve the controller.

\paragraph{Exponential Utility}
We can also opt for a continuous surrogate of the indicator function, in this case the exponential utility $U_t(C) \triangleq \exp(-\frac{1}{\lambda} C)$, where $\lambda > 0$ is a scaling parameter.
Unlike the indicator function, the exponential utility provides nonzero feedback for any given cost and allows us to discriminate between costs (i.e., if $C_1 > C_2$, then $U_t(C_1) < U_t(C_2)$), as shown in~\cref{fig:exponential}.
Furthermore, $\lambda$ acts as a continuous alternative to $C_{t,\max}$ and dictates how quickly or slowly $U_t$ decays to zero, which in a soft way determines the cutoff point for rejecting given costs.

Under this choice, the loss and its gradient become
\begin{small}
\begin{align} \label{eq:exponentiated utility (MPC obj)}
\ell_t(\thetav) &= -\log \Esub{\polv{\thetav}, \hat\fv}{\exp\left( -\frac{1}{\lambda} C(\hat\xv_t, \hat\uv_t) \right)} \\
\nabla \ell_t(\thetav) &= -\frac{\Esub{\polv{\thetav}, \hat\fv}{\exp\left( -\frac{1}{\lambda} C(\hat\xv_t, \hat\uv_t) \right) \nabla_{\thetav} \log \polv{\thetav}(\hat\uv_t)}}
								{\Esub{\polv{\thetav}, \hat\fv}{\exp\left( -\frac{1}{\lambda} C(\hat\xv_t, \hat\uv_t) \right)}}.
\label{eq:grad exponentiated utility (MPC obj))}
\end{align}
\end{small}
The loss function in~\eqref{eq:exponentiated utility (MPC obj)} is also known as the risk-seeking objective in optimal control~\citep{van2010risk}; this classical interpretation is based on a Taylor expansion of~\eqref{eq:exponentiated utility (MPC obj)} showing 
\begin{align*}
\lambda
\ell_t(\thetav) \approx \Esub{\polv{\thetav}, \hat\fv}{C(\hat\xv_t, \uvhat_t)} - \frac{1}{\lambda} \Vbb_{\polv{\thetav}, \hat\fv}[ C(\hat\xv_t, \uvhat_t) ]
\end{align*}
when $\lambda$ is large, where $ \Vbb_{\polv{\thetav}, \hat\fv}[ C(\hat\xv_t, \uvhat_t) ]$ is the variance of $C(\hat\xv_t, \uvhat_t)$.
Here we derive~\eqref{eq:exponentiated utility (MPC obj)} from a different perspective that treats it as a continuous approximation of~\eqref{eq:low cost probability (MPC obj)}.
The use of exponential transformations to approximate indicators is a common machine-learning trick (like the Chernoff bound~\citep{chernoff1952measure}).

\begin{figure}[t]
	\centering
	\begin{subfigure}[b]{0.24\textwidth}
		\includegraphics[width=\textwidth,trim={0 0.5cm 0 0},clip]{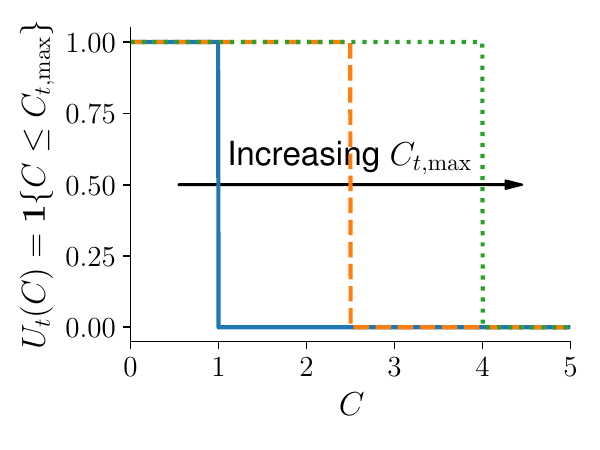}
		\caption{Threshold utility}
		\label{fig:indicator}
	\end{subfigure}
	\begin{subfigure}[b]{0.24\textwidth}
		\includegraphics[width=\textwidth,trim={0 0.5cm 0 0},clip]{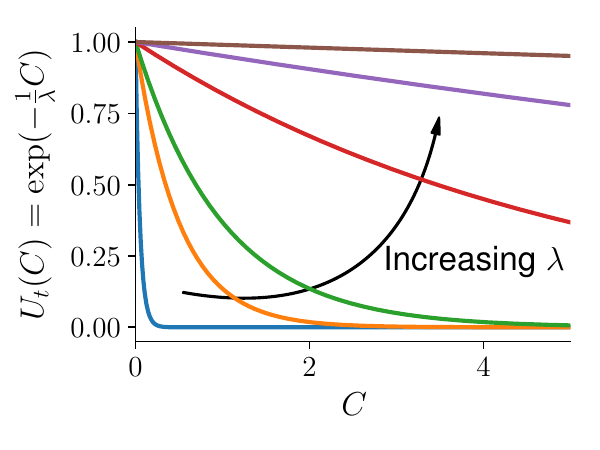}
		\caption{Exponential utility}
		\label{fig:exponential}
	\end{subfigure}
	\caption{Visualization of different utilities.}
	\label{fig:utilities}
\end{figure}

\subsection{Algorithms} \label{sec:algorithms}
We instantiate \ouralg with different choices of loss function, control distribution, and Bregman divergence 
as concrete examples to showcase the flexibility of our framework. 
In particular, we are able to recover well-known MPC algorithms as special cases of Algorithm~\ref{alg:dmd-mpc}.

Our discussions below are organized based on the class of Bregman divergences used in~\eqref{eq:DMD}, and the following algorithms are derived assuming that the control distribution is a sequence of independent distributions.
That is, we suppose $\polv{\thetav}$ is a probability density/mass function that factorizes as 
\begin{align} \label{eq:independent control distribution} 
\polv{\thetav}(\uvhat_t) = \prod_{h=0}^{H-1} \pol{\theta_{h}}(\uhat_{t,h}),
\end{align}
and $\thetav = (\theta_0, \theta_{1}, \ldots, \theta_{H-1})$ for some \emph{basic control distribution} $\pol{\theta}$ parameterized by $\theta \in \Theta$, where $\Theta$ denotes the feasible set for the basic control distribution.
For control distributions in the form of~\eqref{eq:independent control distribution}, 
the shift operator $\Phi$ in~\eqref{eq:shift operation} would set $\tilde{\thetav}_t$ by identifying $\tilde{\theta}_{t,h} = \theta_{t-1,h+1}$ for $h = 0, \dots, H-2$, 
and initializing the final parameter as either $\tilde{\theta}_{t,H-1} = \tilde{\theta}_{t,H-2}$ or $\tilde{\theta}_{t,H-1} = \bar\theta$ for some default parameter $\bar\theta$. 

\subsubsection{Quadratic Divergence}

We start with perhaps the most common Bregman divergence: the quadratic divergence.
That is, we suppose the Bregman divergence in~\eqref{eq:DMD} has a quadratic form
\footnote{This is generated by defining $\psi(\thetav) \triangleq \onehalf \thetav\T \Av \thetav$.}
$
\bregman{\psi}{\thetav}{\thetav'} \triangleq \frac{1}{2}(\thetav-\thetav')\T\Av(\thetav-\thetav')
$ for some positive-definite matrix $\Av$.
Below we discuss different choices of $\Av$ and their corresponding update rules.

\paragraph{Projected Gradient Descent}
This basic update rule is a special case when $\Av$ is the identity matrix. Equivalently, the update can be written as
$\thetav_t = \argmin_{\thetav \in \Thetav}\, \norm{\thetav - (\thetatildev_t - \gamma_t \gv_t)}^2 $. 

\paragraph{Natural Gradient Descent}
We can recover the natural gradient descent algorithm~\citep{Akimoto-CMA} by defining $\Av = \fisher(\tilde\thetav_t)$
where 
\[
\fisher(\tilde\thetav_t) = \Esub{\polv{\tilde\thetav_t}}{\nabla_{\tilde\thetav_t} \log \polv{\tilde\thetav_t}(\uvhat_t) \nabla_{\tilde\thetav_t} \log \polv{\tilde\thetav_t}(\uvhat_t)\T}
\]
is the Fisher information matrix.
This rule uses the natural Riemannian metric of distributions to normalize the effects of different parameterizations of the same distribution~\citep{rattray1998natural}.

\paragraph{Quadratic Problems}
While the above two update rules are quite general, we can further specialize the Bregman divergence to achieve faster learning when the per-round loss function can be shown to be quadratic.
This happens, for instance, when the MPC problem in~\eqref{eq:mpc_obj} is an LQR or LEQR problem\footnote{The dynamics model $\fhat$ is linear, the step cost $c$ is quadratic, the per-round loss $\ell_t$ is~\eqref{eq:expected cost (MPC obj)}, and the basic control distribution is a Dirac-delta distribution.}~\citep{duncan2013linear}. 
That is, if
\[
\ell_t(\thetav) = \onehalf \thetav\T \Rv_t \thetav + \rv_t\T \thetav + \const
\]
for some constant vector $\rv_t$ and positive definite matrix $\Rv_t$,
we can set $\Av = \Rv_t$ and $\gamma_t = 1$, making $\thetav_t$ given by the first step of~\eqref{eq:DMD}  correspond to the optimal solution to $\ell_t$ (i.e., the solution of LQR/LEQR).
The particular values of $\Rv_t$ and $\rv_t$ for each of LQR and LEQR are derived in~\cref{app:lqr and leqr}.

\subsubsection{KL Divergence and the Exponential Family} \label{sec:KL and exponential family}

We show that for control distributions in the exponential family~\citep{ExpFamily}, the Bregman divergence in~\eqref{eq:DMD} can be set to the KL divergence, which is a natural way to measure distances between distributions. 
Toward this end, we review the basics of the exponential family. 
We say a distribution $p_{\eta}$ with natural parameter $\eta$ of random variable $u$ belongs to the \emph{exponential family} if its probability density/mass function satisfies
$	p_{\eta}(u) = \rho(u) \exp \left( \inner{\eta}{\phi(u)} - A(\eta) \right)$, 
where $\phi(u)$ is the sufficient statistics, $\rho(u)$ is the carrier measure, and $
	A(\eta) = \log \int \rho(u) \exp(\inner{\eta}{\phi(u)}) \d{u}
$ is the log-partition function.
The distribution $p_{\eta}$ can also be described by its expectation parameter $\mu \triangleq \Esub{p_{\eta}}{\phi(u)}$, and there is a duality between the two parameterizations:
$
\mu = \nabla A(\eta)
 \text{ and }
\eta = \nabla A^*(\mu)
$, 
where $
A^*(\mu) = \sup_{\eta \in \HH} \, \inner{\eta}{\mu} - A(\eta)
$ is the Legendre transformation of $A$ and $\HH = \{ \eta : A(\eta) < +\infty \}$. That is, $\nabla A = (\nabla A^*)^{-1}$.
The duality results in the property below.
\begin{fact}  \label{fc:KL divergence}
{\normalfont\citep{ExpFamily}}~\;
$
	\KL{p_{\eta}}{p_{\eta'}} = \bregman{A}{\eta'}{\eta} =\bregman{A^*}{\mu}{\mu'}
$.
\end{fact}

We can use~\cref{fc:KL divergence} to define the Bregman divergence in~\eqref{eq:DMD} to optimize a control distribution $\polv{\thetav}$ in the exponential family:
\begin{itemize}
\item if $\thetav$ is an expectation parameter, we can set~\\$\bregman{\psi}{\thetav}{\tilde\thetav_t} \triangleq \KL{\polv{\thetav}}{\polv{\tilde\thetav_t}}$, or
\item if $\thetav$ is a natural parameter,  we can set~\\$\bregman{\psi}{\thetav}{\tilde\thetav_t} \triangleq \KL{\polv{\tilde\thetav_t}}{\polv{\thetav}}$.
\end{itemize}
We demonstrate some examples using this idea below.

\paragraph{Expectation Parameters and Categorical Distributions}

We first discuss the case where $\thetav$ is an expectation parameter and the first step in~\eqref{eq:DMD} is
\begin{align} \label{eq:KL-proximal update with expectation parameter}
\thetav_t  &= \argmin_{\thetav \in \Thetav}\, \inner{\gamma_t \gv_t}{\thetav} + \KL{\polv{\thetav}}{\polv{\tilde\thetav_t}}.
\end{align}

To illustrate, we consider an MPC problem with a \emph{discrete} control space $\{ 1, 2, \ldots, m \}$ and use the categorical distribution as the basic control distribution in~\eqref{eq:independent control distribution}, i.e., we set $\pol{\theta_h} = \cat(\theta_h)$, where $\theta_{h} \in \Delta^m$ is the probability of choosing each control among $\{ 1, 2, \ldots, m \}$ at the $h$\th predicted time step and $\Delta^m$ denotes the probability simplex in $\R^m$.
This parameterization choice makes $\thetav$ an expectation parameter of $\polv{\thetav}$ that corresponds to sufficient statistics given by indicator functions.
With the structure of~\eqref{eq:likelihood-ratio}, the update direction is
\[
g_{t,h} = \Esub{\polv{\tilde\thetav_t}, \hat\fv}{L_t(\hat\xv_t, \hat\uv_t) e_{\uhat_{t,h}} \oslash \tilde\theta_{t,h} } \quad (h = 0, 1, \dots, H-1)
\]
where $\tilde\theta_{t,h}$ and $g_{t,h}$ are the $h$\th elements of $\tilde\thetav_t$ and $\gv_t$, respectively, $e_{\uhat_{t,h}} \in \R^m$ has $0$ for each element except at index $\uhat_{t,h}$ where it is $1$, and $\oslash$ denotes elementwise division.
Update~\eqref{eq:KL-proximal update with expectation parameter} then becomes the exponentiated gradient algorithm~\cite{OCO}:
\begin{small}
\begin{equation}
\theta_{t,h} = \frac{1}{Z_{t,h}} \tilde{\theta}_{t,h} \odot \exp(-\gamma_t g_{t,h}) \quad (h = 0, 1, \ldots, H-1)
\label{eq:eg} 
\end{equation}
\end{small}
where $\theta_{t,h}$ is the $h$\th element of $\thetav_t$, $Z_{t,h}$ is the normalizer for $\theta_{t,h}$, and $\odot$ denotes elementwise multiplication.
That is, instead of applying an additive gradient step to the parameters, the update in~\eqref{eq:KL-proximal update with expectation parameter} exponentiates the gradient and performs elementwise multiplication.
This does a better job of accounting for the geometry of the problem, and makes projection a simple operation of normalizing a distribution.

\paragraph{Natural Parameters and Gaussian Distributions}
Alternatively, we can set $\thetav$ as a natural parameter and use
\begin{align} \label{eq:KL-proximal update with natural parameter}
\thetav_t  &= \argmin_{\thetav \in \Thetav}\, \inner{\gamma_t \gv_t}{\thetav} + \KL{\polv{\tilde\thetav_t}}{\polv{\thetav}}
\end{align}
as the first step in~\eqref{eq:DMD}. 
In particular, we show that, with~\eqref{eq:KL-proximal update with natural parameter}, the structure of the likelihood-ratio derivative in~\eqref{eq:likelihood-ratio} can be leveraged to design an efficient update. 
The main idea follows from the observation that when the gradient is computed through~\eqref{eq:likelihood-ratio} and $\tilde\thetav_t$ is the natural parameter, we can write 
\begin{align} \label{eq:a nice result of natural parameter and likelihood-ratio derivative}
\gv_t = \nabla \ell_t (\tilde{\thetav}_t)
= \Esub{\polv{\tilde\thetav_t}, \hat\fv}{L_t(\hat\xv_t, \uvhat_t) (\phi(\uvhat_t) - \tilde\muv_t)}
\end{align}
where $\tilde\muv_t$ is the expectation parameter of $\tilde\thetav_t$ and $\phi$ is the sufficient statistics of the control distribution.
We combine the factorization in~\eqref{eq:a nice result of natural parameter and likelihood-ratio derivative} with a property of the proximal update below (proven in~\cref{app:proofs}) to derive our algorithm.
\begin{prop}\label{pr:proximal update}
	Let $g_t$ be an update direction.
	Let $\MM$ be the image of $\HH$ under $\nabla A$. 
	If $\mu_t - \gamma_t g_t \in \MM$ and  $
	\eta_{t+1} = \argmin_{\eta \in \HH}\, \inner{\gamma_t g_t}{\eta} + \bregman{A}{\eta}{\eta_t}$, then $\mu_{t+1} = \mu_t - \gamma_t g_t$.\footnote{A similar proposition can be found for~\eqref{eq:KL-proximal update with expectation parameter}.}
\end{prop}
\noindent 
We find that, under the assumption\footnote{If $\mu_t - \gamma_t g_t$ is not in $\MM$, the update in~\eqref{eq:KL-proximal update with natural parameter} needs to perform a projection, the form of which is algorithm dependent.} in \cref{pr:proximal update}, the update rule in~\eqref{eq:KL-proximal update with natural parameter} becomes
\begin{align}\label{eq:a simple update rule (convex form)}
\muv_{t+1} = (1-\gamma_t)\tilde\muv_t + \gamma_t \Esub{\polv{\tilde{\thetav}_t}, \hat\fv}{L_t(\hat\xv_t, \uvhat_t)\phi(\uvhat_t)}.
\end{align}
In other words, when $\gamma_t \in [0,1]$, the update to the expectation parameter $\muv_t$ in~\eqref{eq:DMD} is simply a convex combination of the sufficient statistics and the previous expectation parameter $\tilde\muv_t$.

We provide a concrete example of an MPC algorithm that follows from~\eqref{eq:a simple update rule (convex form)}. Let us consider a continuous control space and use the Gaussian distribution as the basic control distribution in~\eqref{eq:independent control distribution}, i.e., we set 
$\pol{\theta_h}(\uhat_{t,h}) = \NN(\uhat_{t,h}; m_{h}, \Sigma_{h})$ for some mean vector $m_h$ and covariance matrix $\Sigma_h$. For $\pol{\theta_h}$, we can choose sufficient statistics $\phi(\uhat_{t,h}) = (\uhat_{t,h}, \uhat_{t,h}\uhat_{t,h}\T)$, which results in the expectation parameter $\mu_h = (m_h, S_h)$ and the natural parameter $\eta_h = (\Sigma_h^{-1} m_h, -\frac{1}{2} \Sigma_h^{-1} )$, where $S_h \triangleq \Sigma_h + m_h m_h\T$ is the second moment of $\pol{\theta_h}$. 
Let us set $\theta_h$ as the natural parameter. Then~\eqref{eq:KL-proximal update with natural parameter} is equivalent to the update rule for $h=0,\dots, H-1$:
\begin{align} \label{eq:Gaussian update in expectation parameter}
\begin{split}
m_{t,h} &= (1-\gamma_t)\tilde{m}_{t,h} + \gamma_t \Esub{\polv{\thetatildev_t}, \hat\fv}{L_t(\hat\xv_t, \uvhat_t) \uhat_{t,h} } \\
S_{t,h} &= (1-\gamma_t)\tilde{S}_{t,h} + \gamma_t \Esub{\polv{\thetatildev_t}, \hat\fv}{L_t(\hat\xv_t, \uvhat_t) \uhat_{t,h} \uhat_{t,h}\T}.
\end{split}
\end{align}

Several existing algorithms are special cases of~\eqref{eq:Gaussian update in expectation parameter}.
\begin{itemize}
	\item \textit{Cross-entropy method (CEM)}~\cite{CEM}:\\	
	If $\ell_t$ is set to~\eqref{eq:low cost probability (MPC obj)} and 	 $\gamma_t = 1$, then~\eqref{eq:Gaussian update in expectation parameter} becomes 
	\begin{align} \label{eq:update of Gaussian with indicator cost}
	\begin{split}
	m_{t,h} &= \frac{\Esub{\polv{\thetatildev_t}, \hat\fv}{\indicator{C(\hat\xv_t, \uvhat_t) \le C_{t, \max}} \uhat_{t,h}}}{\Esub{\polv{\thetatildev_t}, \hat\fv}{\indicator{C(\hat\xv_t, \uvhat_t) \le C_{t, \max}}}} \\
	S_{t,h} &= \frac{\Esub{\polv{\thetatildev_t}, \hat\fv}{\indicator{C(\hat\xv_t, \uvhat_t) \le C_{t, \max}} \uhat_{t,h} \uhat_{t,h}\T }}{\Esub{\polv{\thetatildev_t}, \hat\fv}{\indicator{C(\hat\xv_t, \uvhat_t) \le C_{t, \max}}}},
	\end{split}
	\end{align}
	which matches the update rule of the cross-entropy method for Gaussian distributions~\cite{CEM}.\footnote{Though CEM is typically presented as updating the mean and \emph{covariance} of a Gaussian distribution, the update rule is derived by matching the first and second moments between the Gaussian distribution and a uniform distribution over trajectories whose costs are at most $C_{t, \max}$, which is identical to~\eqref{eq:update of Gaussian with indicator cost}.}	

\item \textit{Model-predictive path integral (MPPI)}~\cite{Williams-MPPI}:

If we choose $\ell_t$ as the exponential utility, as in~\eqref{eq:exponentiated utility (MPC obj)}, and do not update the covariance, the update rule becomes
\begin{align} \label{eq:update of Gaussian with exponentiated cost}
m_{t,h} = (1 - \gamma_t)\tilde{m}_{t,h}
+ \gamma_t \frac{\Esub{\polv{\thetatildev_t}, \hat\fv}{e^{-\frac{1}{\lambda} C(\hat\xv_t, \uvhat_t)} \uhat_{t,h}}}
				{\Esub{\polv{\thetatildev_t}, \hat\fv}{e^{-\frac{1}{\lambda} C(\hat\xv_t, \uvhat_t)}}},
\end{align}
which reduces to the MPPI update rule~\cite{Williams-MPPI} for $\gamma_t = 1$. This connection is also noted in~\cite{AccMPPI}.
\end{itemize}

\subsection{Extensions}
In the previous sections, we discussed multiple instantiations of \ouralg, showing the flexibility of our framework. But they are by no means exhaustive.
In \cref{app:variations}, we discuss variations of \ouralg, e.g., imposing constraints and different ways to approximate the expectation in~\eqref{eq:likelihood-ratio}.

\section{Related Work} \label{sec:related}
Recent work on MPC has studied sampling-based approaches, which are flexible in that they do not require differentiability of a cost function.
One such algorithm which can be used with general cost functions and dynamics is MPPI, which was proposed by \citet{Williams-MPPI} as a generalization of the control affine case \cite{Williams-Aggressive}.
The algorithm is derived by considering an optimal control distribution defined by the control problem.
This optimal distribution is intractable to sample from, so the algorithm instead tries to bring a tractable distribution (in this case, Gaussian with fixed covariance) as close as possible in the sense of KL divergence.
This ends up being the same as finding the mean of the optimal control distribution. The mean is then approximated as a weighted sum of sampled control trajectories, where the weight is determined by the exponentiated costs.
Although this algorithm works well in practice (including a robust variant~\citep{williams2018robust} achieving state-of-the-art performance in aggressive driving~\citep{drews2019vision}), it is not clear that matching the mean of the distribution should guarantee good performance, such as in the case of a multimodal optimal distribution.
By contrast, our update rule in~\eqref{eq:update of Gaussian with exponentiated cost} results from optimizing an exponential utility.

A closely related approach is the cross-entropy method (CEM) \cite{CEM}, which also assumes a Gaussian sampling distribution but minimizes the KL divergence between the Gaussian distribution and a uniform distribution over low cost samples.
CEM has found applicability in reinforcement learning \cite{Mannor-CEM-Policy-Search,Menache-CEM-TD,NoisyCEM}, motion planning \cite{Helvik-CEM, Kobilarov-CEM-MP}, and MPC \cite{chua2018deep,Williams-IT-MPC}.

These sampling-based control algorithms can be considered special cases of general derivative-free optimization algorithms, such as covariance matrix adaptation evolutionary strategies (CMA-ES) \cite{Hansen-CMA-ES} and natural evolutionary strategies (NES) \cite{Wierstra-NES}.
CMA-ES samples points from a multivariate Gaussian, evaluates their fitness, and adapts the mean and covariance of the sampling distribution accordingly.
On the other hand, NES optimizes the parameters of the sampling distribution to maximize some expected fitness through steepest ascent, where the direction is provided by the natural gradient.
\citet{Akimoto-CMA} showed that CMA-ES can also be interpreted as taking a natural gradient step on the parameters of the sampling distribution. 
As we showed in~\cref{sec:algorithms}, natural gradient descent is a special case of~\ouralg framework.
A similar observation that connects between MPPI and mirror descent  was made by~\citet{AccMPPI}, but their derivation is limited to the KL divergence and Gaussian case.

\section{Experiments} \label{sec:exp}
We use experiments to the validate the flexibility of \ouralg. We show that this framework can handle both continuous (Gaussian distribution) and discrete (categorial distribution) variations of control problems, and that MPC algorithms like MPPI and CEM can be generalized using different step sizes and control distributions to improve performance. Extra details and results are included in~\cref{app:exp details,app:exp results}.

\subsection{Cartpole} \label{sec:cartpole}
We first consider the classic cartpole problem where we seek to swing a pole upright and keep it balanced only using actuation on the attached cart.
We consider both the continuous and discrete control variants.
For the continuous case, we choose the Gaussian distribution as the control distribution and keep the covariance fixed.
For the discrete case, we choose the categorical distribution and use update \eqref{eq:eg}.
In either case, we have access to a biased stochastic model (uses a different pole length compared to the real cart).

\begin{figure}[t!]
	\centering
	\begin{subfigure}[b]{0.49\textwidth}
		\centering
		\includegraphics[width=\textwidth]{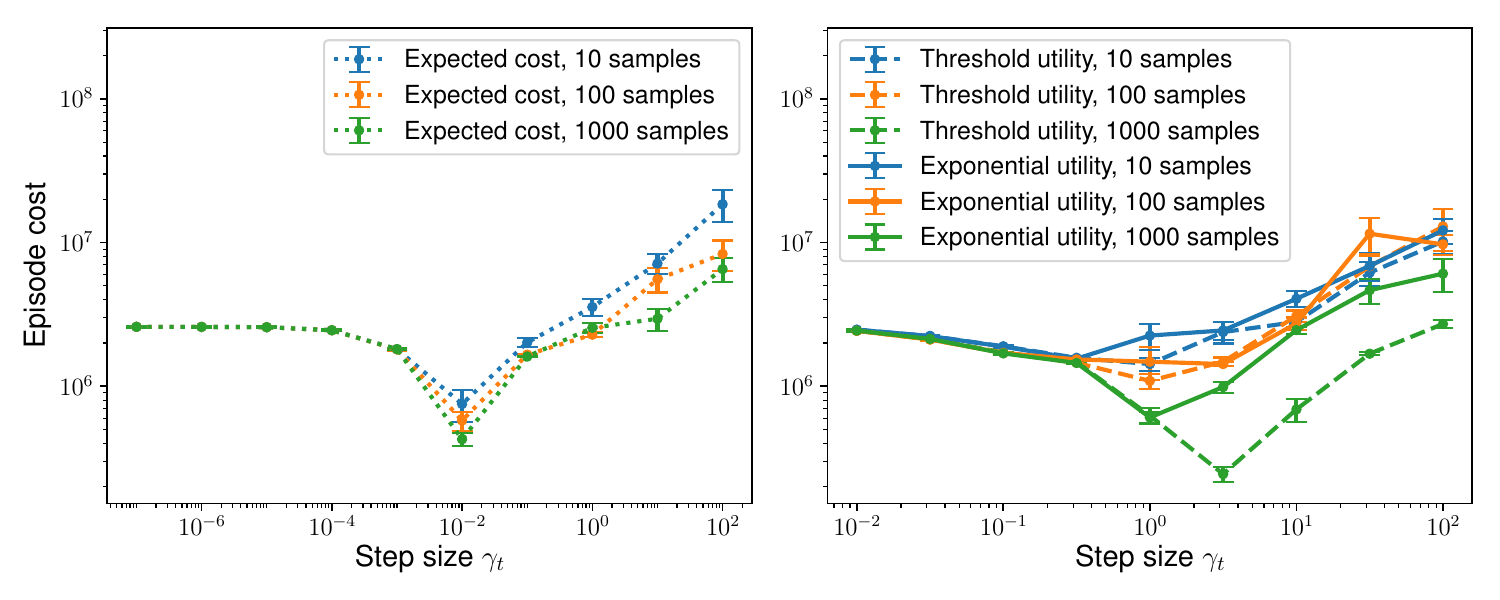}
		\caption{Continuous controls}
		\label{fig:cont_step}
	\end{subfigure}
	\begin{subfigure}[b]{0.49\textwidth}
		\centering
		\includegraphics[width=\textwidth]{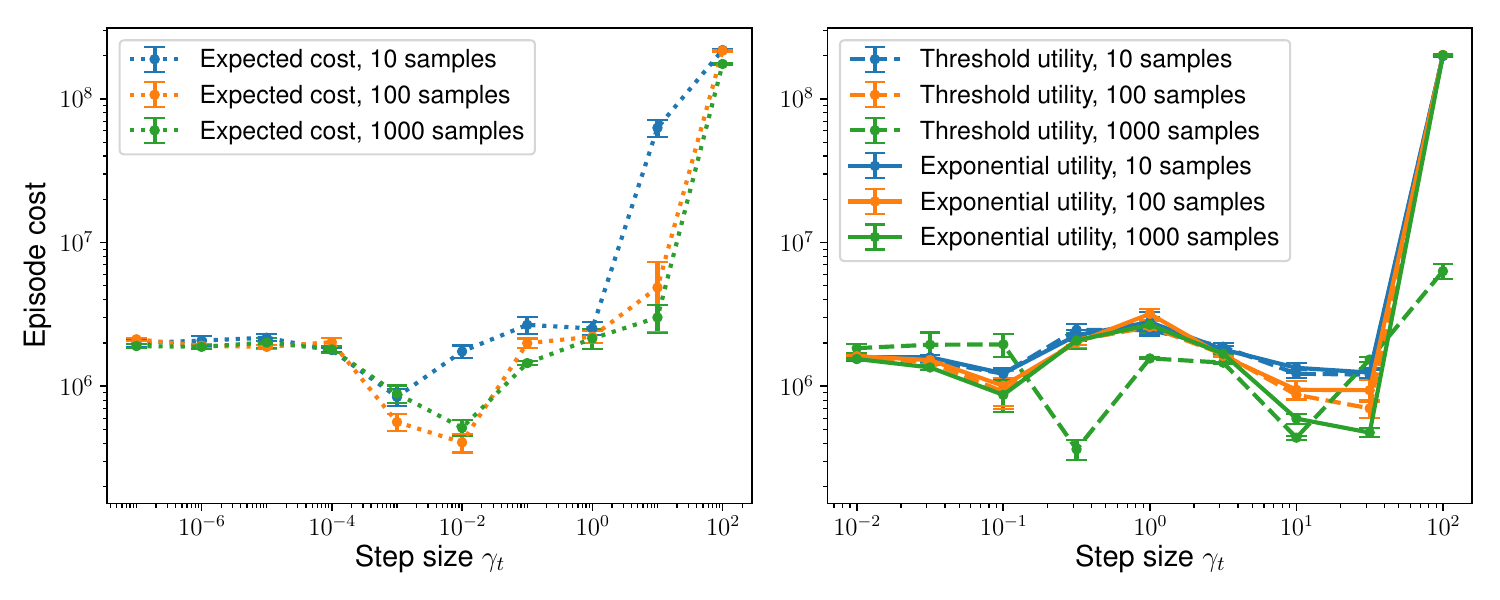}
		\caption{Discrete controls}
		\label{fig:disc_step}
	\end{subfigure}
	\caption{Varying step size and number of samples (same legends for (a) and (b)). EC = expected cost~\eqref{eq:expected cost (MPC obj)}. PLC = probability of low cost~\eqref{eq:low cost probability (MPC obj)} with elite fraction $= 10^{-3}$. EU = exponential utility~\eqref{eq:exponentiated utility (MPC obj)} with $\lambda = 1$.}
	\label{fig:step}
\end{figure}
~
\begin{figure}[h!]
	\centering
	\begin{subfigure}[b]{0.49\textwidth}
		\centering
		\includegraphics[width=\textwidth]{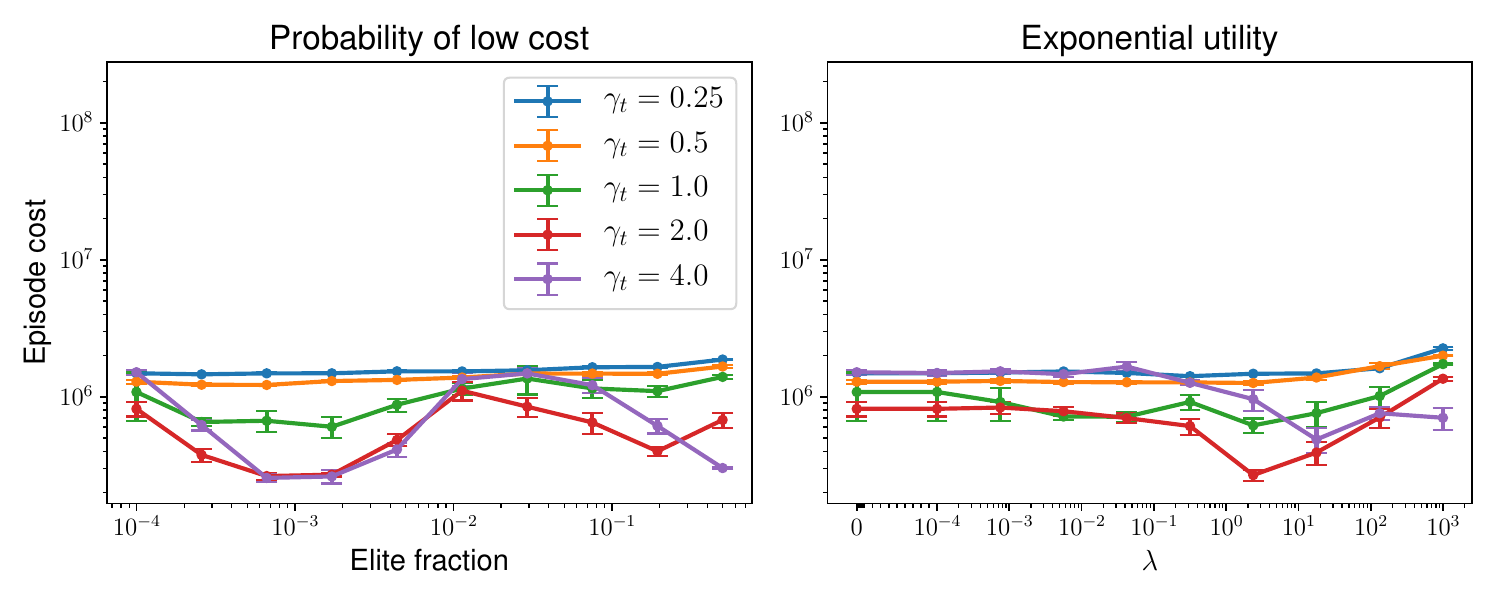}
		\caption{Continuous controls}
		\label{fig:cont_trans}
	\end{subfigure}
	\begin{subfigure}[b]{0.49\textwidth}
		\centering
		\includegraphics[width=\textwidth]{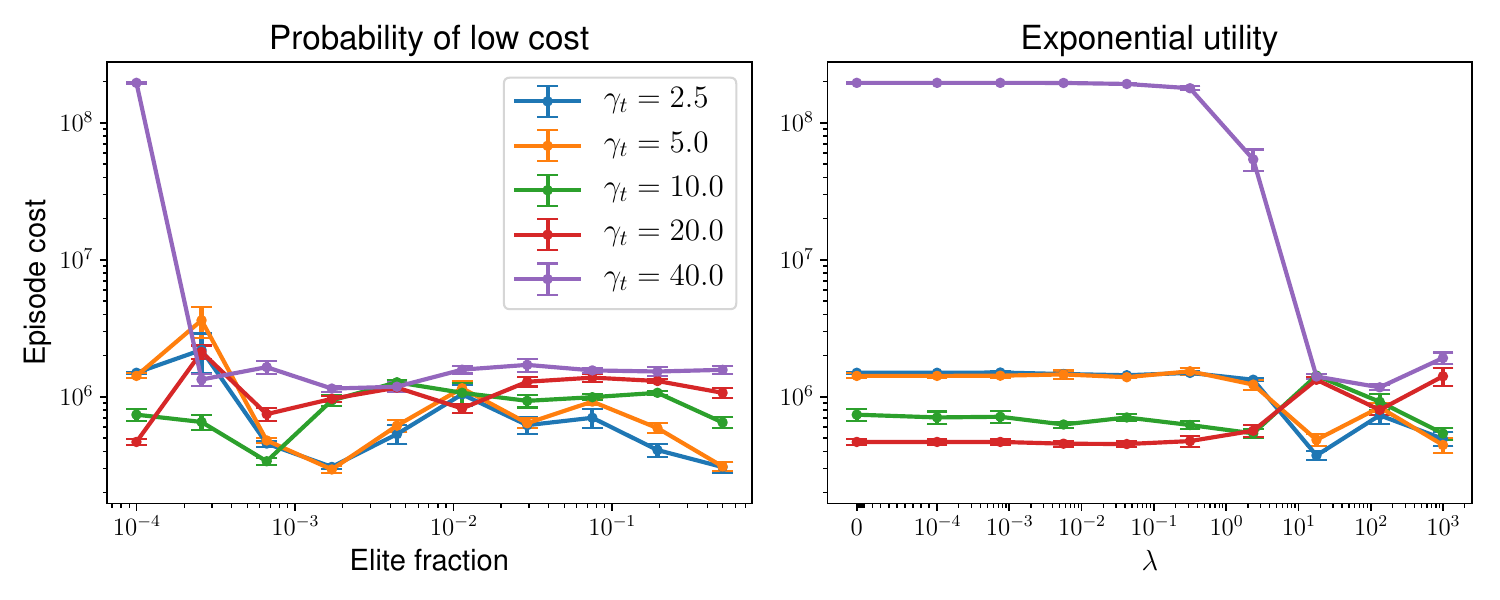}
		\caption{Discrete controls}
		\label{fig:disc_trans}
	\end{subfigure}
	\caption{Varying loss parameter and step size (1000 samples).}
	\label{fig:trans}
\end{figure}

We consider the interaction between the choice of loss,
 step size, and number of samples used to estimate~\eqref{eq:likelihood-ratio},\footnote{For our experiments, we vary the number of samples from $\polv{\thetav}$ and fix the number of samples from $\hat\fv$ to ten. Furthermore, we use common random numbers when sampling from $\hat\fv$ to reduce estimation variance.} shown in~\cref{fig:step,fig:trans}.
For this environment, we can achieve low cost when optimizing the expected cost in~\eqref{eq:expected cost (MPC obj)} with a proper step size ($10^{-2}$ for both continuous and discrete problems) while being fairly robust to the number of samples.
When using either of the utilities, the number of samples is more crucial in the continuous domain, with more samples allowing for larger step sizes.
In the discrete domain~(\cref{fig:disc_step}), performance is largely unaffected by the number of samples when the step size is below $10$, excluding the threshold utility with 1000 samples.
In~\cref{fig:cont_trans}, for a large range of utility parameters, we see that using step sizes above $1$ (the step size set in MPPI and CEM) give significant performance gains.
In~\cref{fig:disc_trans}, there's a more complicated interaction between the utility parameter and step size, with huge changes in cost when altering the utility parameter and keeping the step size fixed.

\subsection{AutoRally} \label{sec:autorally}
\subsubsection{Platform Description}

We use the autonomous AutoRally platform~\citep{AutoRally} to run a high-speed driving task on a dirt track, with the goal of the task to achieve as low a lap time as possible.
The robot~(\cref{fig:car}) is a 1:5 scale RC chassis capable of driving over $20~\mathrm{m/s}$ ($45~\mathrm{mph}$) and has a desktop-class Intel Core i7 CPU and Nvidia GTX 1050 Ti GPU.
Our code for the control algorithm is based on modifications of code available on the AutoRally repository.\footnote{\href{https://github.com/AutoRally/autorally}{\color{blue}\texttt{https://github.com/AutoRally/autorally}}}
For real-world experiments, we estimate the car's pose using a particle filter from~\citep{drews2019vision} which relies on a monocular camera, IMU, and GPS.
In both simulated and real-world experiments, the dynamics model is a neural network which has been fitted to data collected from human demonstrations.
We note that the dynamics model is deterministic, so we don't need to estimate any expectations with respect to the dynamics.

\begin{figure}
	\centering
	\includegraphics[width=0.4\textwidth]{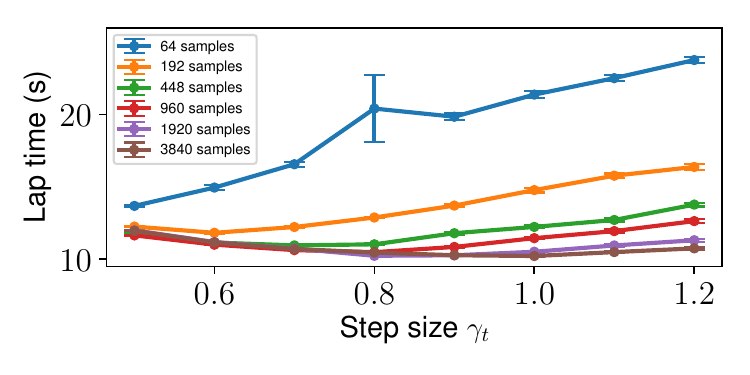}
	\caption{Simulated AutoRally performance with different step sizes and number of samples. Though many samples coupled with large steps  yield the smallest lap times, the performance gains are small past 1920 samples. With fewer samples, a lower step size helps recover some lost performance.}
	\label{fig:gazebo comparison}
\end{figure}

\subsubsection{Simulated Experiments} \label{subsec:simulated experiments}

We first use the Gazebo simulator~(\cref{fig:gazebo} in~\cref{app:autorally details}) from the AutoRally repo to perform a sweep of algorithm parameters, particularly the step size and number of samples, to evaluate how changing these parameters can affect the performance of~\ouralg. 
For all of the experiments, the control distribution is a Gaussian with fixed covariance, and we use update~\eqref{eq:update of Gaussian with exponentiated cost} (i.e., the loss is the exponential utility~\eqref{eq:exponentiated utility (MPC obj)}) with $\lambda = 6.67$.
The resulting lap times are shown in~\cref{fig:gazebo comparison}.\footnote{The large error bar for 64 samples and step size of 0.8 is due to one particular lap where the car stalled at a turn for about 60 seconds.}
We see that although using more samples does result in smaller lap times, there are diminishing returns past 1920 samples per gradient.
Indeed, with a proper step size, even as few as 192 samples can yield lap times within a couple seconds of 3840 samples and a step size of 1.
We also observe that the curves converge as the step size decreases further, implying that only a certain number of samples are needed for a given step size.
This is a particularly important advantage of~\ouralg over methods like MPPI: by changing the step size,~\ouralg can perform much more effectively with fewer samples, making it a good choice for embedded systems which can't produce many samples due to computational constraints.

\begin{figure}[t]
	\centering
	\includegraphics[width=0.3\textwidth]{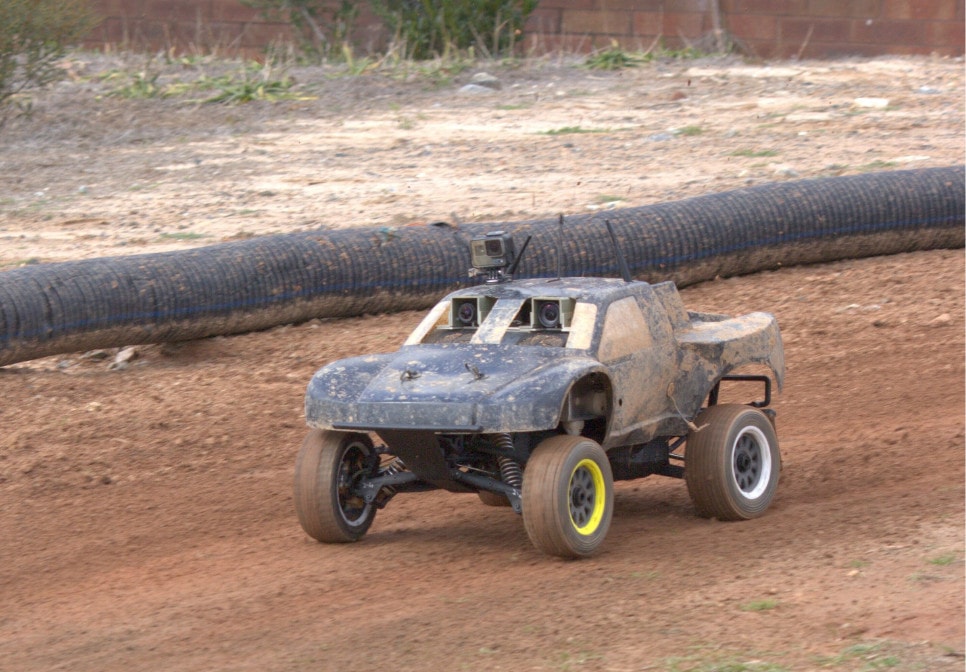}	
	\caption{Rally car driving during an experiment.}
	\label{fig:car}	
\end{figure}

\begin{figure}
	\centering
	\includegraphics[width=0.4\textwidth]{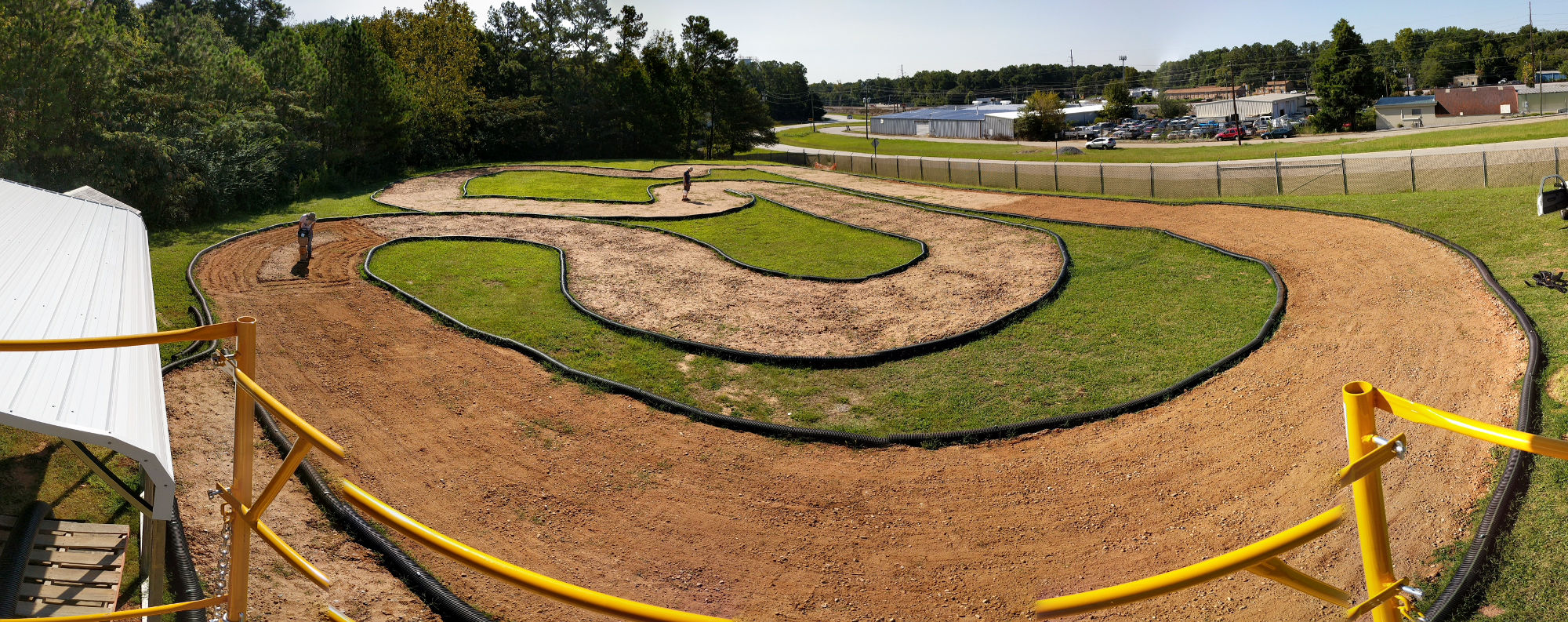}
	\caption{Real-world AutoRally task.}
	\label{fig:real}
\end{figure}

\begin{table*}[ht]
	\centering
	\caption{Statistics for real-world experiments at target of $9~\mathrm{m/s}$.}
	\label{tab:real world slow}
	\begin{tabular}{c|c V{3} c|c|c}
		Samples & Step size $\gamma_t$ & Lap time ($\mathrm{s}$) & Avg. speed ($\mathrm{m/s}$) & Max speed ($\mathrm{m/s}$) \\ \specialrule{0.15em}{0em}{0em}
		$1920$  & $1$                  & $31.76 \pm 0.55$        & $5.70 \pm 0.16$             & $9.21 \pm 0.30$            \\ \cline{2-5} 
		& $0.8$                & $31.81 \pm 0.21$                & $5.75 \pm 0.03$             & $9.03 \pm 0.19$    \\ \cline{2-5}
		& $0.6$                & $32.83 \pm 0.31$        & $5.60 \pm 0.05$             & $8.62 \pm 0.12$            \\ \specialrule{0.15em}{0em}{0em}
		$64$    & $1$                  & $33.74 \pm 0.78$        & $5.45 \pm 0.16$             & $9.50 \pm 0.22$            \\ \cline{2-5}
		& $0.8$                & $33.84 \pm 0.80$                & $5.46 \pm 0.11$             & $9.12 \pm 0.26$            \\ \cline{2-5} 
		& $0.6$                & $33.61 \pm 0.74$        & $5.50 \pm 0.13$             & $9.14 \pm 0.42$           
	\end{tabular}
\end{table*}
\begin{table*}[ht]
	\centering
	\caption{Statistics for real-world experiments at target of $11~\mathrm{m/s}$.}
	\label{tab:real world fast}
	\begin{tabular}{c|c V{3} c|c|c}
		Samples               & Step size $\gamma_t$ & Lap time ($\mathrm{s}$) & Avg. speed ($\mathrm{m/s}$) & Max speed ($\mathrm{m/s}$) \\ \specialrule{0.15em}{0em}{0em}
		$64$                  & $1$                  & $31.05 \pm 0.67$        & $5.80 \pm 0.26$             & $10.17 \pm 0.30$           \\ \cline{2-5} 
		& $0.6$                & $30.30 \pm 0.56$        & $5.98 \pm 0.15$             & $10.30 \pm 0.05$          
	\end{tabular}
\end{table*}

\subsubsection{Real-World Experiments}
In the real-world setting~(\cref{fig:real}), the control distribution is a Gaussian with fixed covariance, and we use update~\eqref{eq:update of Gaussian with exponentiated cost} with $\lambda = 8$.
We ran two sets of experiments, each with a different target speed: one at $9~\mathrm{m/s}$ and the other at $11~\mathrm{m/s}$.\footnote{The conference version of this paper does not have the second set of experiments (target of $11~\mathrm{m/s}$).
Those experiments were conducted after the camera-ready deadline of the conference.}

For the first set of experiments, we used the following configurations: each of 1920 and 64 samples, and each of step sizes 1 (corresponding to MPPI), 0.8, and 0.6.\footnote{Due to weaker batteries used with 64 samples, results should not be compared across number of samples.}
Overall~(\cref{tab:real world slow}), there's a mild degradation in performance when decreasing the step size at 1920 samples, due to the car taking a longer path on the track~(\cref{fig:1920-1.0} vs.~\cref{fig:1920-0.6} in~\cref{app:real world}).
With 64 samples, the results seem unaffected by the step size.
This could be because, despite the noisiness of the DMD-MPC update, the setpoint controller in the car's steering servo acts as a filter, smoothing out the control signal and allowing the car to drive on a consistent path~(\cref{fig:speed_64} in~\cref{app:real world}). Videos of this experiment can be found at \href{https://youtu.be/vZST3v0_S9w}{\color{blue}\texttt{https://youtu.be/vZST3v0\_S9w}}.

For the second set of experiments, we fixed the number of samples at 64 and used step sizes of 1 (corresponding to MPPI) and 0.6.
The statistics slightly improve with a decreased step size, but qualitatively there is a larger difference between the step sizes.
With a step size of 1, the car often wobbles while driving, turns around at one point, and crashes in one of the trials~(\cref{fig:fast-64-1.0}).
On the other hand, with a step size of 0.6, the car drives much more smoothly and achieves the aggressive driving task with no issues~(\cref{fig:fast-64-0.6}).
Despite the smoothing effect of the low-level controllers in the car, the more stringent costs associated with the larger target speed cause the noisiness of the DMD-MPC update to manifest in the car's performance when using a step size of 1.
A smaller step size mitigates this noisiness.
Videos of this experiment can be found at \href{https://youtu.be/MhuqiHo2t98}{\color{blue}\texttt{https://youtu.be/MhuqiHo2t98}}.

\section{Conclusion} \label{sec:conclusion}
We presented a connection between model predictive control and online learning.
From this connection, we proposed an algorithm based on dynamic mirror descent that can work for a wide variety of settings and cost functions.
We also discussed the choice of loss function within this online learning framework and the sort of preference each loss function imposes.
From this general algorithm and assortment of loss functions, we show several well known algorithms are special cases and presented a general update for members of the exponential family.

We empirically validated our algorithm on continuous and discrete simulated problems and on a real-world aggressive driving task. In the process, we also studied the parameter choices within the framework, finding, for example, that in our framework a smaller number of rollout samples can be compensated for by varying other parameters like the step size.

We hope that the online learning and stochastic optimization viewpoints of MPC presented in this paper opens up new possibilities for using tools from these domains, such as alternative efficient sampling techniques~\citep{bellman1971differential} and accelerated optimization methods~\citep{AccMD,AccMPPI}, to derive new MPC algorithms that perform well in practice.

\section*{Acknowledgements}
This material is based upon work supported by NSF NRI award 1637758, NSF CAREER award 1750483, an NSF Graduate Research Fellowship under award No. 2015207631, and a National Defense Science \& Engineering Graduate Fellowship.
We thank Aravind Battaje, Nathan Hatch, and Hemanth Sarabu for assisting in AutoRally experiments.
~\\~\\~\\
\bibliographystyle{plainnat}
\bibliography{dmdmpc}

\clearpage
\onecolumn
\appendices

\section{Shift Operator} \label{app:shift operator}

We discuss some details in defining the shift operator.
Let $\thetav_{t-1}$ be the approximate solution to the previous problem and $\tilde{\thetav}_{t}$ denote the initial condition of $\thetav$ in solving~\eqref{eq:mpc_obj}, and consider sampling $\uvhat_t \sim \polv{\tilde\thetav_t}$ and $\uvhat_{t-1} \sim \polv{\thetav_{t-1}}$.
We set 
\[
\tilde{\thetav}_{t} = \Phi(\thetav_{t-1})
\]
by defining a \emph{shift operator} $\Phi$ that outputs a new parameter in $\Thetav$. This $\Phi$ can be chosen to satisfy desired properties, one example being that when conditioned on $\uhat_{t-1}$ and $x_t$, 
the marginal distributions of $\uhat_{t},\dots, \uhat_{t+H-2}$ are the same for both $\uvhat_t$ of  $\polv{\tilde{\thetav}_t}$ and $\uvhat_{t-1}$ of $\polv{\thetav_{t-1}}$.
A simple example of this property is shown in~\cref{fig:shift}.
Note that $\uvhat_t$ 
also involves a new control $ \uhat_{t+H-1}$ that is not in $\uvhat_{t-1}$,
so the choice of $\Phi$ is not unique but algorithm dependent; for example, we can set $ \uhat_{t+H-1}$ of $\polv{\tilde\thetav_{t}}$ to follow the same distribution as $ \uhat_{t+H-2}$ (cf.~\cref{sec:algorithms}).
Because the subproblems in~\eqref{eq:mpc_obj} of two consecutive time steps share all control variables except for the first and the last ones, the ``shifted'' parameter $\Phi(\thetav_{t-1})$ to the current problem should be almost as good as the optimized parameter $\thetav_{t-1}$ is to the previous problem. In other words, setting  $\tilde{\thetav}_{t} = \Phi(\thetav_{t-1})$ provides a warm start to~\eqref{eq:mpc_obj} and amortizes the computational complexity of solving for $\thetav_t$.

\begin{figure}[h!]
	\centering
	\includegraphics[width=0.75\textwidth]{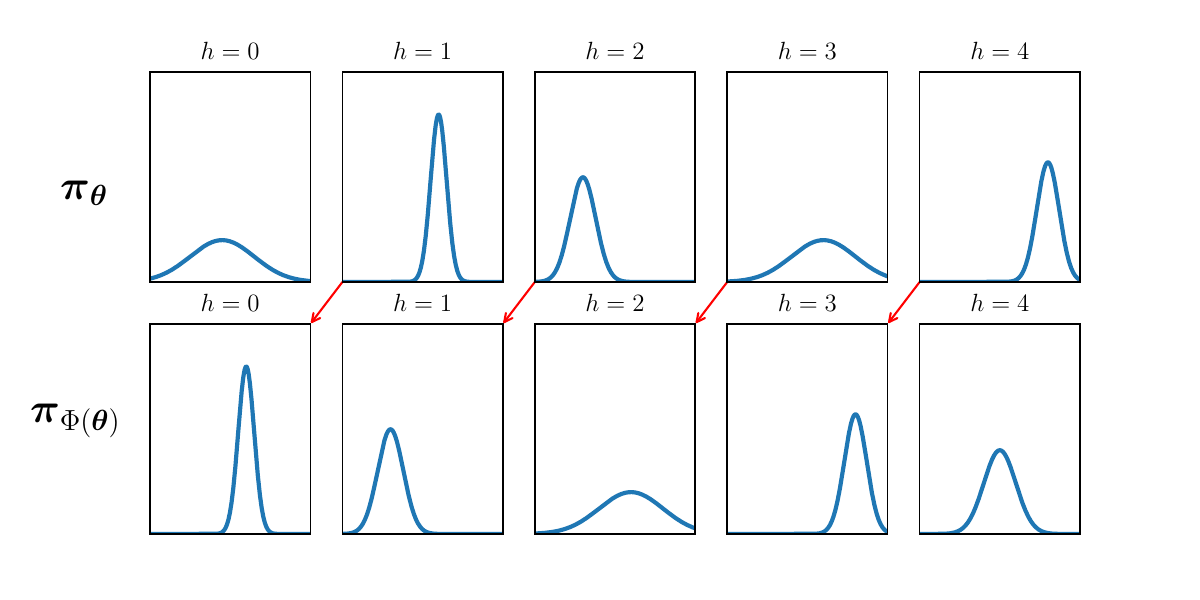}
	\vspace{-3mm}
	\caption{A simple example of the shift operator $\Phi$. Here, the control distribution $\polv{\thetav}$ consists of a sequence of $H = 5$ independent Gaussian distributions.
		The shift operator moves the parameters of the Gaussians one time step forward and replaces the parameters at $h = 4$ with some default parameters.}
	\label{fig:shift}
\end{figure}

\section{Variations of \ouralg} \label{app:variations}
The control distributions in \ouralg can be fairly general (in addition to the categorical and Gaussian distributions that we discussed) and control constraints on the problem (e.g., control limits) can be directly incorporated through proper choices of control distributions, such as the beta distribution, or through mapping the unconstrained control through some squashing function (e.g., $\tanh$ or clamp).
Though our framework cannot directly handle state constraints as in constrained optimization approaches, a constraint can be relaxed to an indicator function which activates if the constraint is violated.
The indicator function can then be added to the cost function in~\eqref{eq:sum of costs} with some weight that encodes how strictly the constraint should be enforced.

Moreover, different integration techniques, such as Gaussian quadrature~\citep{bellman1971differential}, can be adopted to replace the likelihood-ratio derivative in~\eqref{eq:likelihood-ratio} for computing the required gradient direction. 
We also note that the independence assumption on the control distribution in~\eqref{eq:independent control distribution} is not necessary in our framework; time-correlated control distributions and feedback policies are straightforward to consider in \ouralg.

\clearpage
\section{Proofs}\label{app:proofs}

\begin{proof}[Proof of \cref{pr:proximal update}]
	We prove the first statement; the second one follows directly from the duality relationship.
	The statement follows from the derivations below; we can write
	\begin{align*}
	\eta_{t+1} &= \argmin_{\eta \in \HH}\,  \inner{\gamma_t g_t}{\eta} + \bregman{A}{\eta}{\eta_t} \\
	&= \argmin_{\eta \in \HH}\, \inner{\gamma_t  g_t}{\eta} + A(\eta) - \inner{\nabla A(\eta_t)}{\eta} \\
	&= \argmin_{\eta \in \HH}\, \inner{\gamma_t  g_t - \mu_t}{\eta} + A(\eta) \\
	&= \argmax_{\eta \in \HH}\, \inner{\mu_t - \gamma_t  g_t}{\eta} - A(\eta) \\
	&= \nabla A^*(\mu_t - \gamma_t  g_t)
	\end{align*}
	where the last equality is due to the assumption that $\mu_t - \gamma_t  g_t \in \MM$. Then applying $\nabla A$ on both sides and using the relationship that $\nabla A = (\nabla A^*)^{-1}$, we have 
	$
	\mu_{t+1} = \nabla A(\eta_{t+1}) = \mu_t - \gamma_t  g_t
	$.
\end{proof} 

\section{Derivation of LQR and LEQR Losses} \label{app:lqr and leqr}
The dynamics in Equation \eqref{eq:true dynamics} are given by
\[
x_{t+1} = A x_t + B u_t + w_t
\]
for some matrices $A \in \R^{n \times n}$ and $B \in \R^{n \times m}$ and $w_t \sim \N(0, W)$, where $W \in \S_{++}^n$. For a control sequence $\uvhat_t$, noise sequence $\wvhat_t$, and initial state $x_t$, the resulting state sequence $\xvhat_t$ is found through convolution:
\[
\begin{bmatrix}
\xhat_t \\ \xhat_{t+1} \\ \xhat_{t+2} \\ \vdots \\ \xhat_{t+H}
\end{bmatrix}
= \begin{bmatrix} I \\ A \\ A^2 \\ \vdots \\ A^H \end{bmatrix} x_t
+ \begin{bmatrix} 0 & 0 & \cdots & 0 \\
B & 0 & \cdots & 0 \\
AB & B & \cdots & 0 \\
\vdots & \vdots & \ddots & \vdots \\
A^{H-1}B & A^{H-2}B & \cdots & B
\end{bmatrix} \begin{bmatrix} \uhat_t \\ \uhat_{t+1} \\ \vdots \\ \uhat_{t+H-1} \end{bmatrix}
+ \begin{bmatrix} 0 & 0 & \cdots & 0 \\
I & 0 & \cdots & 0 \\
A & I & \cdots & 0 \\
\vdots & \vdots & \ddots & \vdots \\
A^{H-1} & A^{H-2} & \cdots & I
\end{bmatrix} \begin{bmatrix} \what_t \\ \what_{t+1} \\ \vdots \\ \what_{t+H-1} \end{bmatrix},
\]
or, in matrix form:
\[
\xvhat_t = \Fv x_t + \Gv \uvhat_t + \Lv \wvhat_t,
\]
where $\Fv$, $\Gv$, and $\Lv$ are defined naturally from the convolution equation above.
Note that $\wvhat_t \sim \N(0, \Wv)$, where \mbox{$\Wv = \diag(W, W, \ldots, W, W)$}.
Thus, we also have that
\[
\xvhat_t \sim \N(\Fv x_t + \Gv \uvhat_t, \Lv \Wv \Lv\T).
\]

We define the instantaneous and terminal costs as
\begin{align*}
c(x, u) &= \onehalf x\T Q x + \onehalf u\T R u \\
c_{\mathrm{end}}(x) &= \onehalf x\T Q_{\mathrm{end}} x,
\end{align*}
where $Q, Q_{\mathrm{end}} \in \S_+^n$ and $R \in \S_{++}^m$. Thus, the statistic $C(\xvhat_t, \uvhat_t)$ is
\[
C(\xvhat_t, \uvhat_t) = \onehalf \xvhat_t\T \Qv \xvhat_t + \onehalf \uvhat_t\T \Rv \uvhat_t,
\]
where $\Qv = \diag(Q, Q, \ldots, Q, Q_{\mathrm{end}})$ and $\Rv = \diag(R, R, \ldots, R, R)$.

Our control distribution is a Dirac delta distribution located at the given parameter: $\piv_{\thetav}(\uvhat_t) = \delta(\uvhat_t - \thetav)$.
\subsection{LQR}
The loss is defined as $\ell_t(\thetav) = \Esub{\polv{\thetav}, \hat\xv_t}{\onehalf \xvhat_t\T \Qv \xvhat_t + \onehalf \uvhat_t\T \Rv \uvhat_t}$. 
Expanding this out gives:
\begin{align*}
\ell_t(\thetav) &= \Esub{\polv{\thetav}, \hat\xv_t}{\onehalf \xvhat_t\T \Qv \xvhat_t + \onehalf \uvhat_t\T \Rv \uvhat_t} \\
&= \onehalf \thetav\T \left( \Gv\T \Qv \Gv + \Rv \right) \thetav + x_t\T \Fv\T \Qv \Gv \thetav + \onehalf x_t\T \Fv\T \Qv \Fv x_t + \onehalf \E{\wvhat_t\T \Lv\T \Qv \Lv \wvhat_t} \\
&= \onehalf \thetav\T \left( \Gv\T \Qv \Gv + \Rv \right) \thetav + x_t\T \Fv\T \Qv \Gv \thetav + \onehalf x_t\T \Fv\T \Qv \Fv x_t + \onehalf \tr{\Qv \Lv \Wv \Lv\T}.
\end{align*}
We see this is a quadratic problem in $\thetav$ by defining
\begin{align*}
\Rv_t &= \Gv\T \Qv \Gv + \Rv \\
\rv_t &= \Gv\T \Qv \Fv x_t.
\end{align*}
\subsection{LEQR}
The loss is defined as
\[
\ell_t(\thetav) = -\log \Esub{\polv{\thetav}, \hat\xv_t}{\exp\left( -\frac{1}{\lambda} \left( \onehalf \xvhat_t\T \Qv \xvhat_t + \onehalf \uvhat_t\T \Rv \uvhat_t \right) \right)}
\]
for some parameter $\lambda > 0$.
For compactness, we define $\Qv' = \frac{1}{\lambda} \Qv$ and $\Rv' = \frac{1}{\lambda} \Rv$ so that the exponent contains \mbox{$-\onehalf \xvhat_t\T \Qv' \xvhat - \onehalf \uvhat_t\T \Rv' \uvhat_t$}.
In expanding the loss, we use the following fact:
\begin{fact}\label{fact:expectation of exponential}
For $x \sim \N(\mu, \Sigma)$, where $\Sigma \in \S_{++}^n$, and constants $A \in \S_+^n$ and $b \in \R^n$:
\[
\Esub{x}{\exp\left( -\onehalf x\T A x - b\T x \right)} = \frac{1}{\sqrt{|A\Sigma + I|}} \exp\left( -\onehalf \left( \mu\T \Sigma\inv \mu - (\Sigma\inv \mu - b)\T (A + \Sigma\inv)\inv (\Sigma\inv \mu - b) \right) \right).
\]
\end{fact}
\begin{proof}
We expand the expectation and complete the square:
\begin{align*}
\Esub{x}{\exp\left( -\onehalf x\T A x - b\T x \right)} &= \frac{1}{\sqrt{(2\pi)^n |\Sigma|}} \int \exp\left( -\onehalf (x - \mu)\T \Sigma\inv (x - \mu) \right) \exp\left(-\onehalf x\T Ax - b\T x \right) \d{x} \\
&= \frac{1}{\sqrt{(2\pi)^n |\Sigma|}} \int \exp\left( -\onehalf \left[ x\T(A + \Sigma\inv)x + 2(b - \Sigma\inv \mu)\T x + \mu\T \Sigma\inv \mu \right] \right) \d{x} \\
&= \frac{1}{\sqrt{(2\pi)^n |\Sigma|}} \exp(c) \int \exp\left( -\onehalf (x - \tilde{\mu})\T \tilde{\Sigma}\inv (x - \tilde{\mu}) \right) \d{x} \\
&= \frac{\sqrt{(2\pi)^n |\tilde{\Sigma}|}}{\sqrt{(2\pi)^n |\Sigma|}} \exp(c) \\
&= \frac{1}{\sqrt{|A + \Sigma\inv| |\Sigma|}} \exp(c) \\
&= \frac{1}{\sqrt{|A\Sigma + I|}} \exp\left( -\onehalf \left( \mu\T \Sigma\inv \mu - (\Sigma\inv \mu - b)\T (A + \Sigma\inv)\inv (\Sigma\inv \mu - b) \right) \right),
\end{align*}
where $\tilde{\mu} = (A + \Sigma\inv)\inv (\Sigma\inv \mu - b)$, $\tilde{\Sigma} = (A + \Sigma\inv)\inv$, and $c = -\onehalf \left(\mu\T \Sigma\inv \mu - (\Sigma\inv \mu - b)\T (A + \Sigma\inv)\inv (\Sigma\inv \mu - b) \right)$.
\end{proof}

We now expand the loss:
\begin{align*}
\ell_t(\thetav) &= -\log \Esub{\polv{\thetav}, \hat\xv_t}{\exp\left( -\onehalf \xvhat_t\T \Qv' \xvhat_t - \onehalf \uvhat_t\T \Rv' \uvhat_t \right)} \\
&= -\log \Esub{\hat\xv_t}{\exp\left( -\onehalf \xvhat_t\T \Qv' \xvhat_t - \onehalf \thetav\T \Rv' \thetav \right)} \\
&= -\log \Bigg \{ \frac{1}{\sqrt{|\Qv' \Lv\Wv\Lv\T + I|}} \exp\bigg(-\onehalf \big[ (\Fv x_t + \Gv\thetav)\T (\Lv\Wv\Lv\T)\inv (\Fv x_t + \Gv\thetav) \\
& \qquad~\qquad~\qquad~\qquad~\qquad~\qquad\quad\qquad~\; -(\Fv x_t + \Gv \thetav)\T (\Lv\Wv\Lv\T\Qv'\Lv\Wv\Lv\T + \Lv\Wv\Lv\T)\inv (\Fv x_t + \Gv \thetav) \\
& \qquad~\qquad~\qquad~\qquad~\qquad~\qquad\quad\qquad~\;+ \thetav\T \Rv' \thetav \big] \bigg) \Bigg\} \\
&= \onehalf \left[ (\Fv x_t + \Gv \thetav)\T [ (\Lv\Wv\Lv\T)\inv + (\Lv\Wv\Lv\T\Qv'\Lv\Wv\Lv\T + \Lv\Wv\Lv\T)\inv ] (\Fv x_t + \Gv \thetav) + \thetav\T \Rv' \thetav \right] + \onehalf \log |\Qv' \Lv\Wv\Lv\T + I|. \\ 
\end{align*}
We see this is a quadratic problem in $\thetav$ by defining
\begin{align*}
\Rv_t &= \Gv\T \left[ (\Lv\Wv\Lv\T)\inv + \left( \frac{1}{\lambda}\Lv\Wv\Lv\T\Qv \Lv\Wv\Lv\T + \Lv\Wv\Lv\T \right)\inv \right] \Gv + \frac{1}{\lambda} \Rv \\
\rv_t &= \Gv\T \left[ (\Lv\Wv\Lv\T)\inv + \left( \frac{1}{\lambda}\Lv\Wv\Lv\T\Qv \Lv\Wv\Lv\T + \Lv\Wv\Lv\T \right)\inv \right] \Fv x_t.
\end{align*}

\section{Experimental Setup} \label{app:exp details}

\subsection{Cartpole}
The state is $x_t = (p_t, \varphi_t, v_t, \dot{\varphi}_t)$, where $p_t$ is the cart position, $\varphi_t$ is the pole's angle, $v_t$ and $\dot{\varphi}_t$ are the corresponding velocities, and the control $u_t$ is the force applied to the cart.
We define the instantaneous cost and terminal cost of the MPC problem as
\begin{align*}
c(x_t, u_t) &= 10p_t^2 + 500(\varphi_t - \pi)^2 + v_t^2 + 15 \dot{\varphi}_t^2 + 1000 \cdot \indicator{|\varphi_t - \pi| \ge \Delta} \\
c_{\mathrm{end}}(x_t) &= c(x_t, 0)
\end{align*}
where $\Delta$ is some threshold. For our experiments, we set $\Delta = 12^\circ = 0.21$ radians.

In our experiments, the pole is massless except for some weight at the end of the pole.
The mass of the cart and pole weight are $0.711~\mathrm{kg}$ and $0.209~\mathrm{kg}$, respectively. The true length of the pole is $0.326~\mathrm{m}$, whereas the length used in the model is $0.346~\mathrm{m}$.
Each time step is modeled using an Euler discretization of $0.02$ seconds.
Each episode of the problem lasts 500 time steps (i.e, 10 seconds) and has episode cost equal to the sum of encountered instantaneous costs.
Both the true system and the model apply Gaussian additive noise to the commanded control with zero mean and a standard deviation of $5$ newtons.
For the continuous system, the commanded control is clamped to $\pm 25$ newtons.
For the discrete system, the controller can either command $10$ newtons to the left, $10$ newtons to the right, or $0$ newtons.

Both the discrete and continuous controller use a planning horizon of 50 time steps (i.e., 1 second).
For the continuous controller, we keep the standard deviation of the Gaussian distribution fixed at $2$ newtons for each time step in the planning horizon. When applying a control $u_t$ on the real cartpole, we choose the mode of $\pi_{\theta_t}$ rather than sample from the distribution.

All reported results were gathered using ten episodes per parameter setting.

\subsection{AutoRally} \label{app:autorally details}
The state of the vehicle is $x_t = (p_{x,t}, p_{y,t}, \varphi_t, r_t, v_{x,t}, v_{y,t}, \dot{\varphi}_t)$, where $(p_{x,t}, p_{y,t})$ is the position of the car in the global frame, $\varphi_t$ and $r_t$ are the yaw and roll angles, $v_{x,t}$ and $v_{y,t}$ are the longitudinal and lateral velocities in the car frame, and $\dot{\varphi}_t$ is the yaw rate.
The control $u_t$ we apply is the throttle and steering angle.
For some weights $w_1, \ldots, w_4$, the cost function is
\begin{align*}
c(x_t, u_t) &= w_1 |s_t - s_{\mathrm{tgt}}|^k + w_2 M(p_{x,t}, p_{y,t}) + w_3 S_c(x_t) \\
c_{\mathrm{end}}(\xv_t) &= w_4 C(\xv_t).
\end{align*}
Here, $s_t$ and $s_\mathrm{tgt}$ are the current and target speed of the car, respectively.
Note the speed is calculated as \mbox{$s_t = \sqrt{v_{x,t}^2 + v_{y,t}^2}$}.
$M(p_{x,t}, p_{y,t})$ is the positional cost of the car (low cost in center of track, high cost at edge of track), $S_c(x_t)$ is an indicator variable which activates if the slip angle\footnote{The slip angle is defined as $-\arctan \frac{v_{y,t}}{|v_{x,t}|}$, which gives the angle between the direction the car is pointing and the direction in which it is actually traveling.} exceeds a certain threshold, and $C(\xv_t)$ is an indicator function which activates if the car leaves the track at all in the trajectory.
Note that the terminal cost depends on the trajectory instead of the terminal state.
Each time step represents $0.02$ seconds for every experiment except the real-world experiment with a target of $11~\mathrm{m/s}$ where each time step represents $0.025$ seconds.
The length of the planning trajectory is 100 time steps (i.e., either 2 seconds or 2.5 seconds depending on the length of the time step).
The values for the cost function parameters are given in Table~\ref{tab:autorally params}.

The control space for each of the throttle and steering angle is normalized to the range $[-1, 1]$.
For our experiments, we clamp the throttle to $[-1, 0.65]$.
In simulated experiments, the standard deviations of the throttle and steering angle distributions were $0.3$ and $0.275$, respectively.
In the real world experiments, they were both set to $0.3$. 
When applying a control $u_t$ on the car, we chose the mean of $\pi_{\theta_t}$ rather than sampling from the distribution.

\begin{figure}[H]
	\centering
	\includegraphics[width=0.5\textwidth]{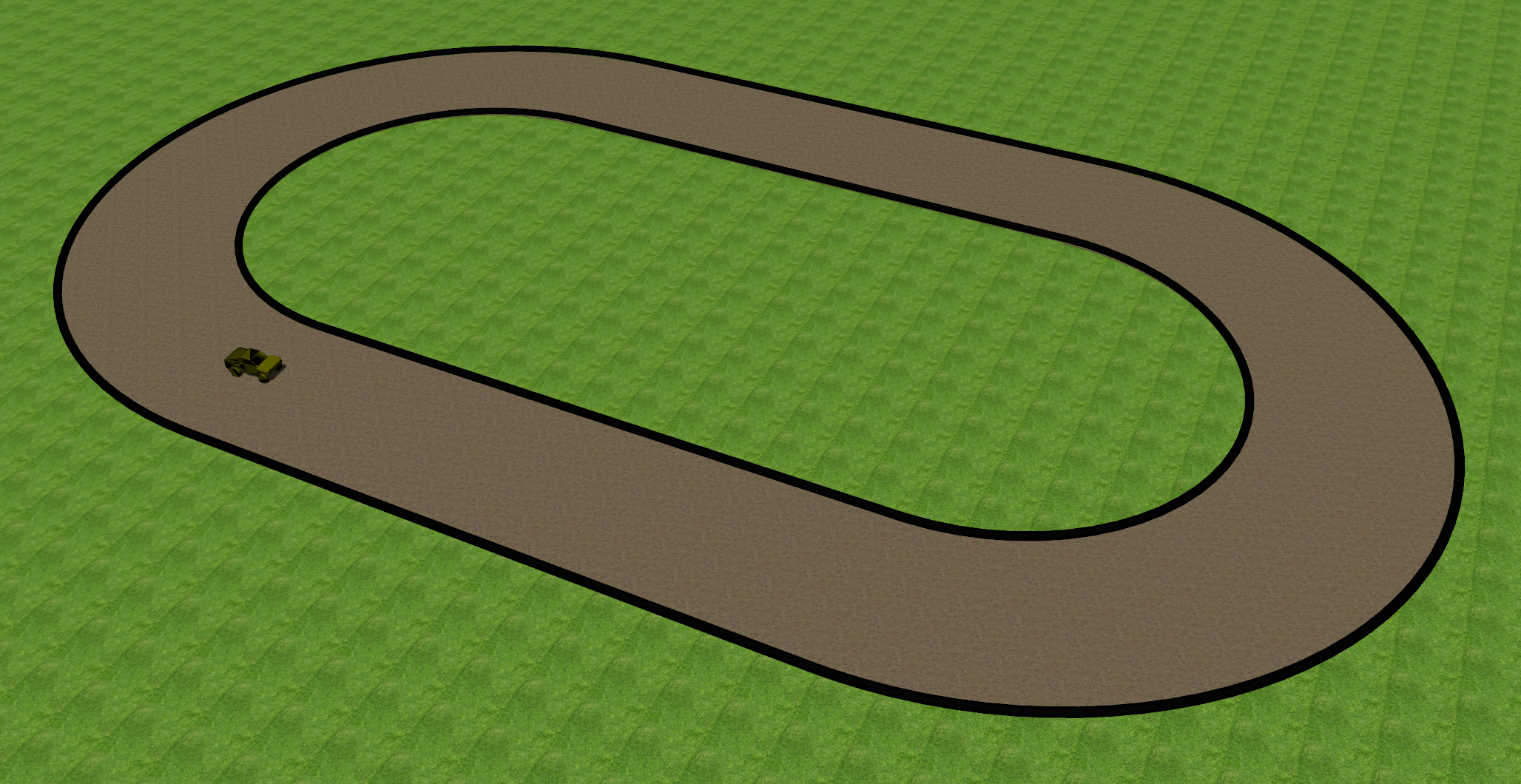}
	\caption{Simulated AutoRally task.}
	\label{fig:gazebo}
\end{figure}

In simulation, the environment~(\cref{fig:gazebo}) is an elliptical track approximately 3 meters wide and 30 meters across at its furthest point.
The real-world dirt track is about 5 meters wide and and has a track length of 170 meters.
All reported results for simulated experiments were gathered using 30 consecutive laps in the counter-clockwise direction for each parameter setting.
For real-world experiments, results were gathered using ten laps for each parameter setting when the target speed is $9~\mathrm{m/s}$ and five laps for $11~\mathrm{m/s}$.

\begin{table}[h!]
\caption{Cost function settings for AutoRally experiments.}
\label{tab:autorally params}
\centering
\begin{tabular}{c|c|c|c|c|c|c|c}
                 & $s_\mathrm{tgt}$ ($\mathrm{m/s}$) & $k$ & $w_1$ & $w_2$ & $w_3$ & $w_4$   & Slip angle threshold ($\mathrm{rad}$) \\ \hline
Gazebo simulator & $11$                              & $1$ & $30$  & $250$ & $10$  & $10000$ & $0.275$                               \\ \hline
Real world       & $9$ or $11$                       & $2$ & $4.25$& $200$ & $100$ & $10000$ & $0.9$
\end{tabular}
\end{table}

\section{Extra Experimental Results} \label{app:exp results}
\subsection{Simulated Experiments}
Adding onto the results from~\cref{subsec:simulated experiments}, we qualitatively evaluate two particular extremes: \mbox{few vs. many samples (64 vs. 3840)} and small vs. large step size (0.5 vs. 1) by looking at the path and speed of the car during the episode~(\cref{fig:speed}).
At small step sizes~(\cref{fig:speed_64_0.5,fig:speed_3840_0.5}), the path and speed profiles are rather similar, while with few samples and a large step size~(\cref{fig:speed_64_1.0}), the car drives much more slowly and erratically, sometimes even stopping.
In the ideal scenario with many samples and a large step size, the car can achieve consistently high speed while driving smoothly~(\cref{fig:speed_3840_1.0}).

We also experimented with instead optimizing the expected cost~\eqref{eq:expected cost (MPC obj)} and found performance was dramatically worse~(\cref{fig:speed identity}), even when using 3840 samples per gradient.
At best, the car would drive in the center of the track at speeds below $4~\mathrm{m/s}$~(\cref{fig:speed identity 0.075}), and at worst, the car would either slowly drive along the track walls~(\cref{fig:speed identity 0.025}) or the controller would eventually produce $\mathrm{NaN}$ controls that would prematurely end the experiment~(\cref{fig:speed identity 0.1}).
This poor performance is likely due to most samples in the estimate of~\eqref{eq:grad expected cost (MPC obj)} having very high cost (e.g., due to leaving the track) and contributing significantly to the gradient estimate.
On the other hand, when estimating~\eqref{eq:grad exponentiated utility (MPC obj))}, as in the experiments in~\cref{subsec:simulated experiments}, these high cost trajectories are assigned very low weights so that only low cost trajectories contribute to the gradient estimate.

\begin{figure}[p]
	\centering
	\begin{subfigure}[b]{0.25\textwidth}
		\centering
		\includegraphics[width=\textwidth]{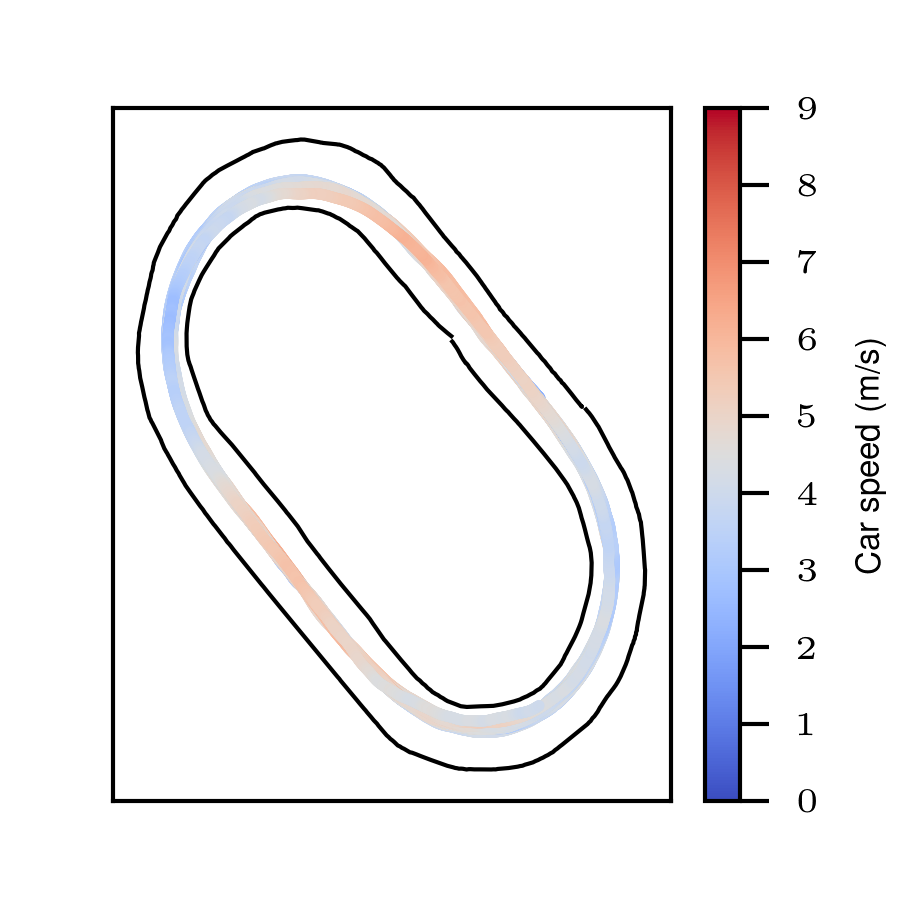}
		\caption{64 samples, $\gamma_t = 0.5$}
		\label{fig:speed_64_0.5}
	\end{subfigure}
	\begin{subfigure}[b]{0.25\textwidth}
		\centering
		\includegraphics[width=\textwidth]{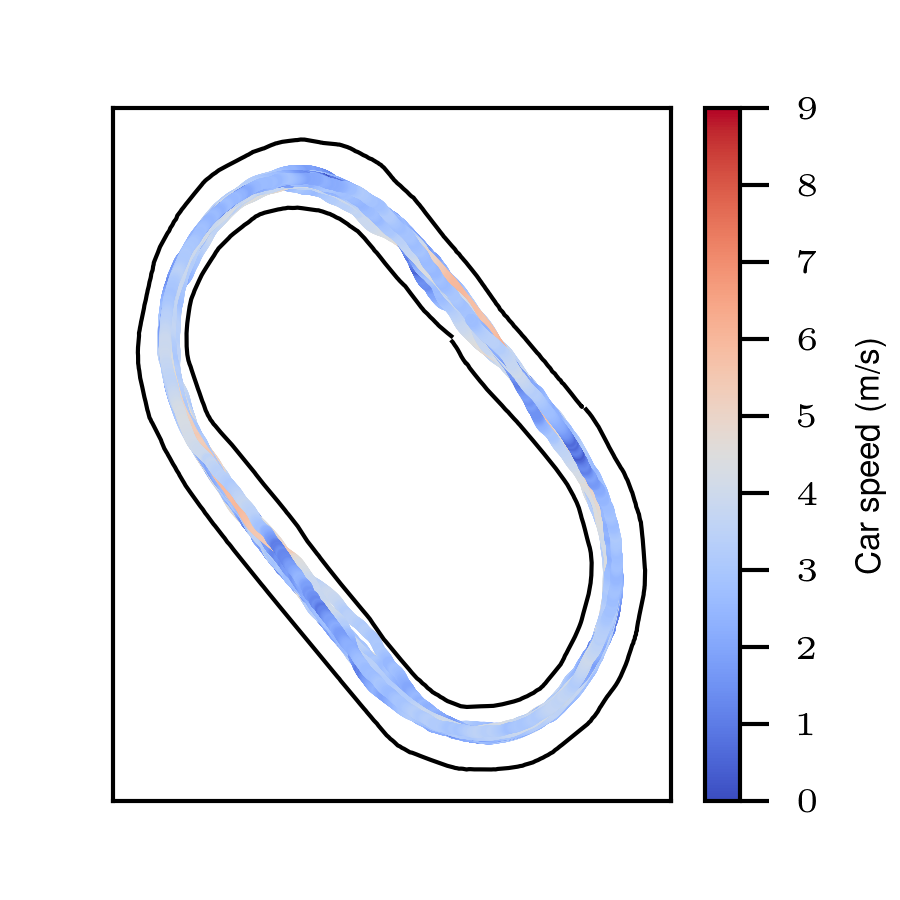}
		\caption{64 samples, $\gamma_t = 1$}
		\label{fig:speed_64_1.0}
	\end{subfigure}\\
	\begin{subfigure}[b]{0.25\textwidth}
		\centering
		\includegraphics[width=\textwidth]{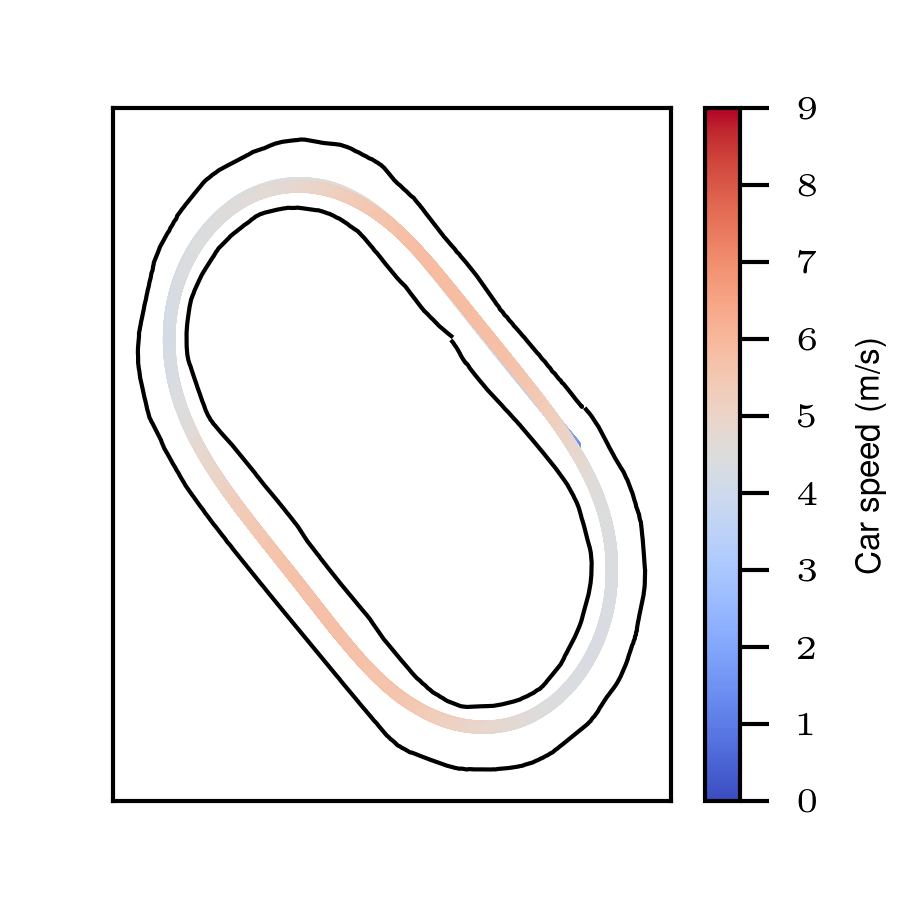}
		\caption{3840 samples, $\gamma_t = 0.5$}
		\label{fig:speed_3840_0.5}
	\end{subfigure}
	\begin{subfigure}[b]{0.25\textwidth}
		\centering
		\includegraphics[width=\textwidth]{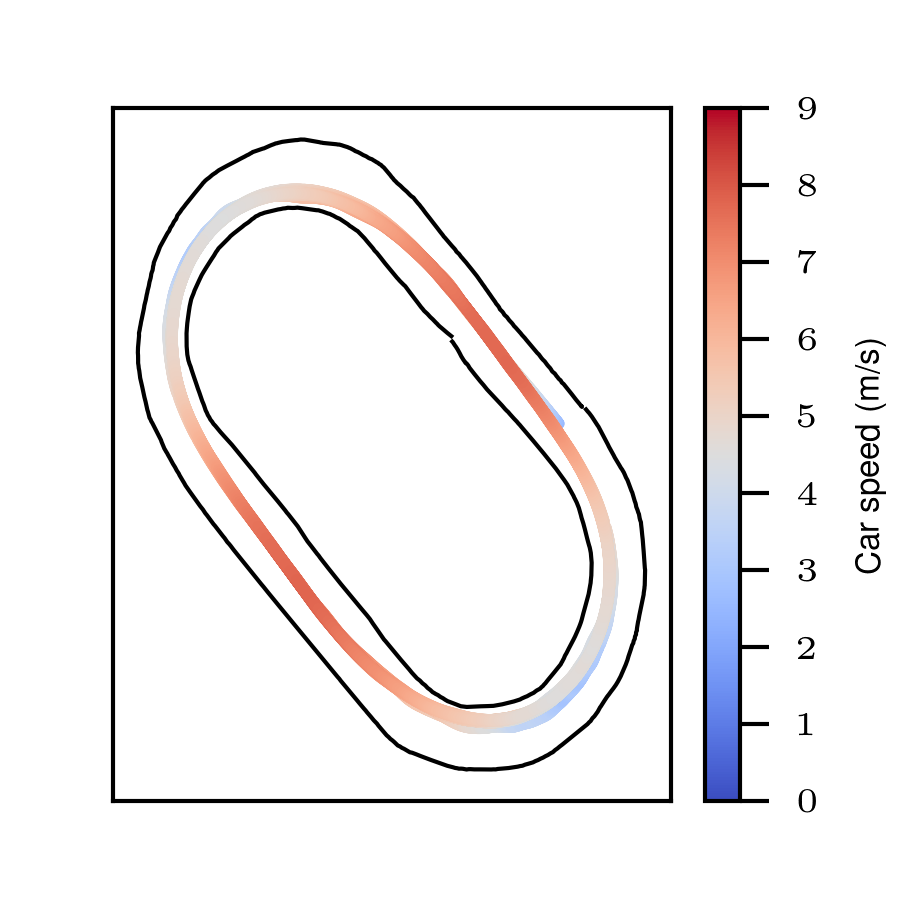}
		\caption{3840 samples, $\gamma_t = 1$}
		\label{fig:speed_3840_1.0}
	\end{subfigure}
	\caption{Car speeds when optimizing the exponential utility~\eqref{eq:exponentiated utility (MPC obj)}. The speeds and trajectories are very similar at step size $0.5$, irrespective of the number of samples. At step size $1$, though, 64 samples result in capricious maneuvers and low speeds, whereas 3840 samples result in smooth driving at high speeds.}
	\label{fig:speed}
\end{figure}
\begin{figure}[p]
	\centering
	\begin{subfigure}[b]{0.25\textwidth}
		\centering
		\includegraphics[width=\textwidth]{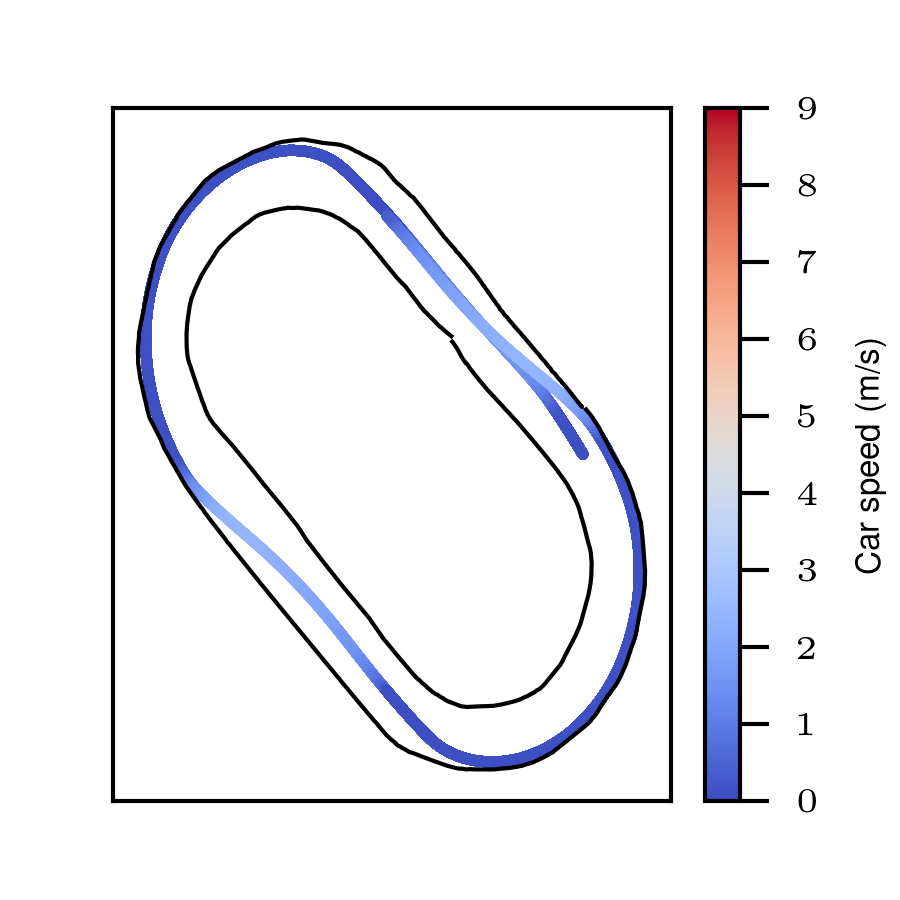}
		\caption{$\gamma_t = 0.025$}
		\label{fig:speed identity 0.025}
	\end{subfigure}
	\begin{subfigure}[b]{0.25\textwidth}
		\centering
		\includegraphics[width=\textwidth]{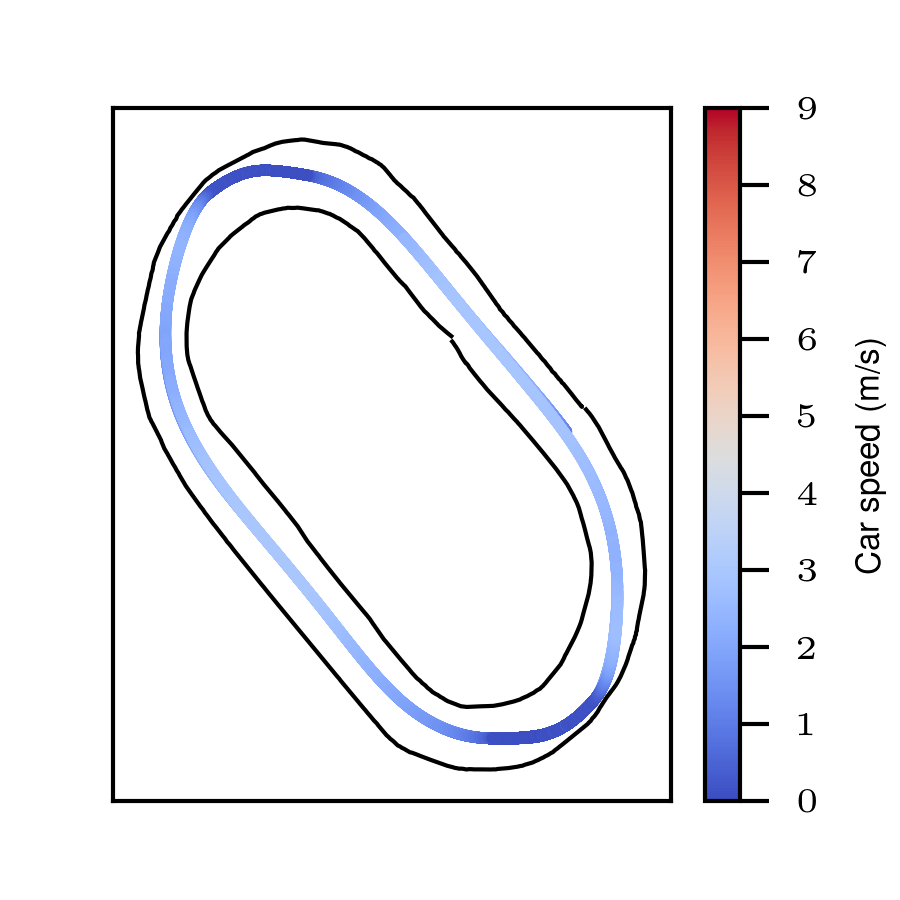}
		\caption{$\gamma_t = 0.05$}
		\label{fig:speed identity 0.05}
	\end{subfigure}\\
	\begin{subfigure}[b]{0.25\textwidth}
		\centering
		\includegraphics[width=\textwidth]{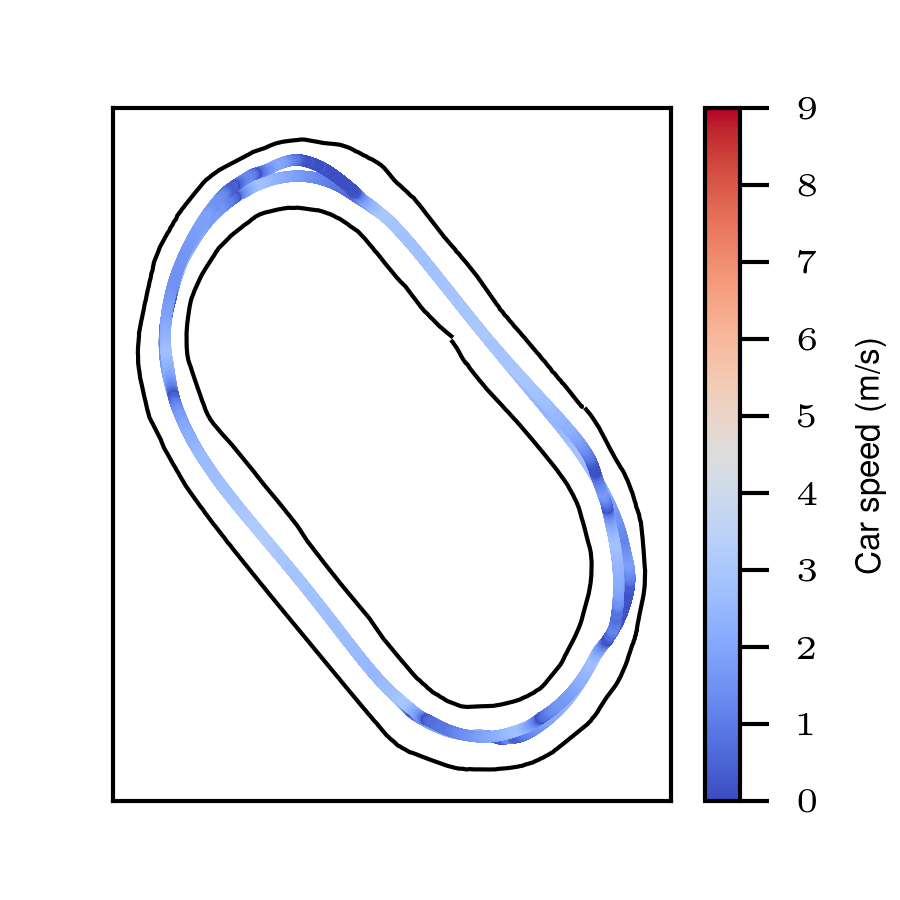}
		\caption{$\gamma_t = 0.075$}
		\label{fig:speed identity 0.075}
	\end{subfigure}
	\begin{subfigure}[b]{0.25\textwidth}
		\centering
		\includegraphics[width=\textwidth]{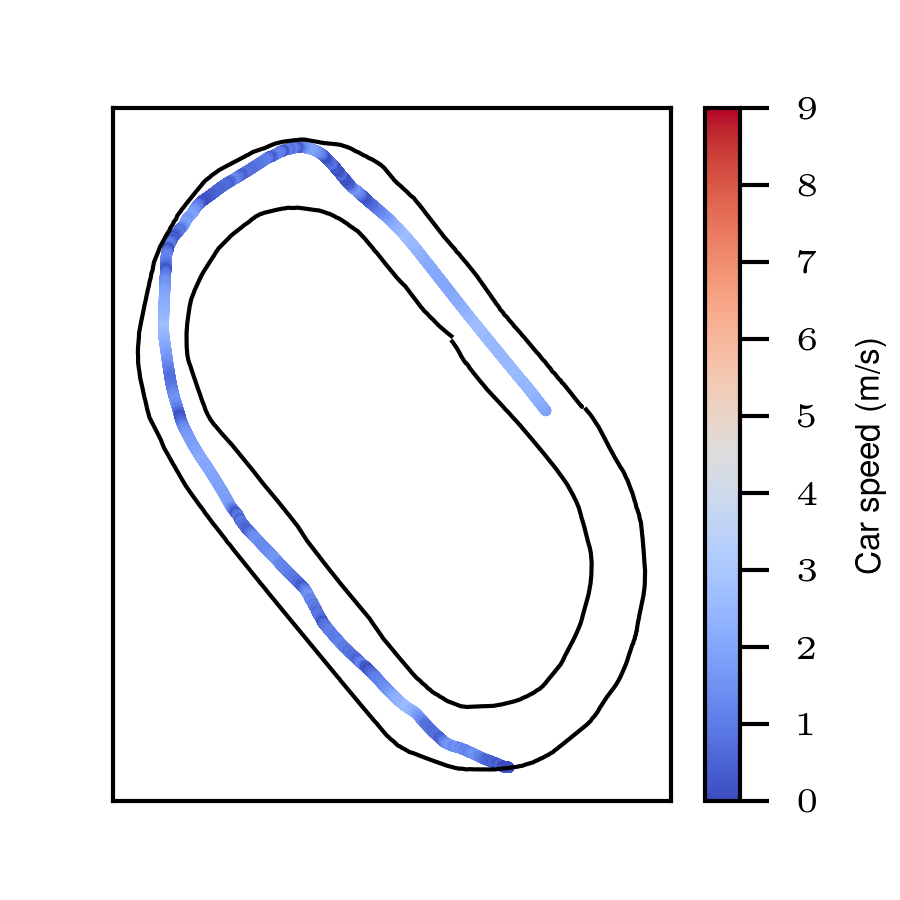}
		\caption{$\gamma_t = 0.1$}
		\label{fig:speed identity 0.1}
	\end{subfigure}
	\caption{Car speeds when optimizing the expected cost~\eqref{eq:expected cost (MPC obj)}. All tested step sizes result in low speeds. At too low or too high of a step size, the car will drive along the wall or crash into it.}
	\label{fig:speed identity}
\end{figure}

\clearpage
\subsection{Figures for Real-World Experiments} \label{app:real world}
\begin{figure}[h!]
	\centering
	\begin{subfigure}[b]{0.3\textwidth}
		\centering
		\includegraphics[width=\textwidth]{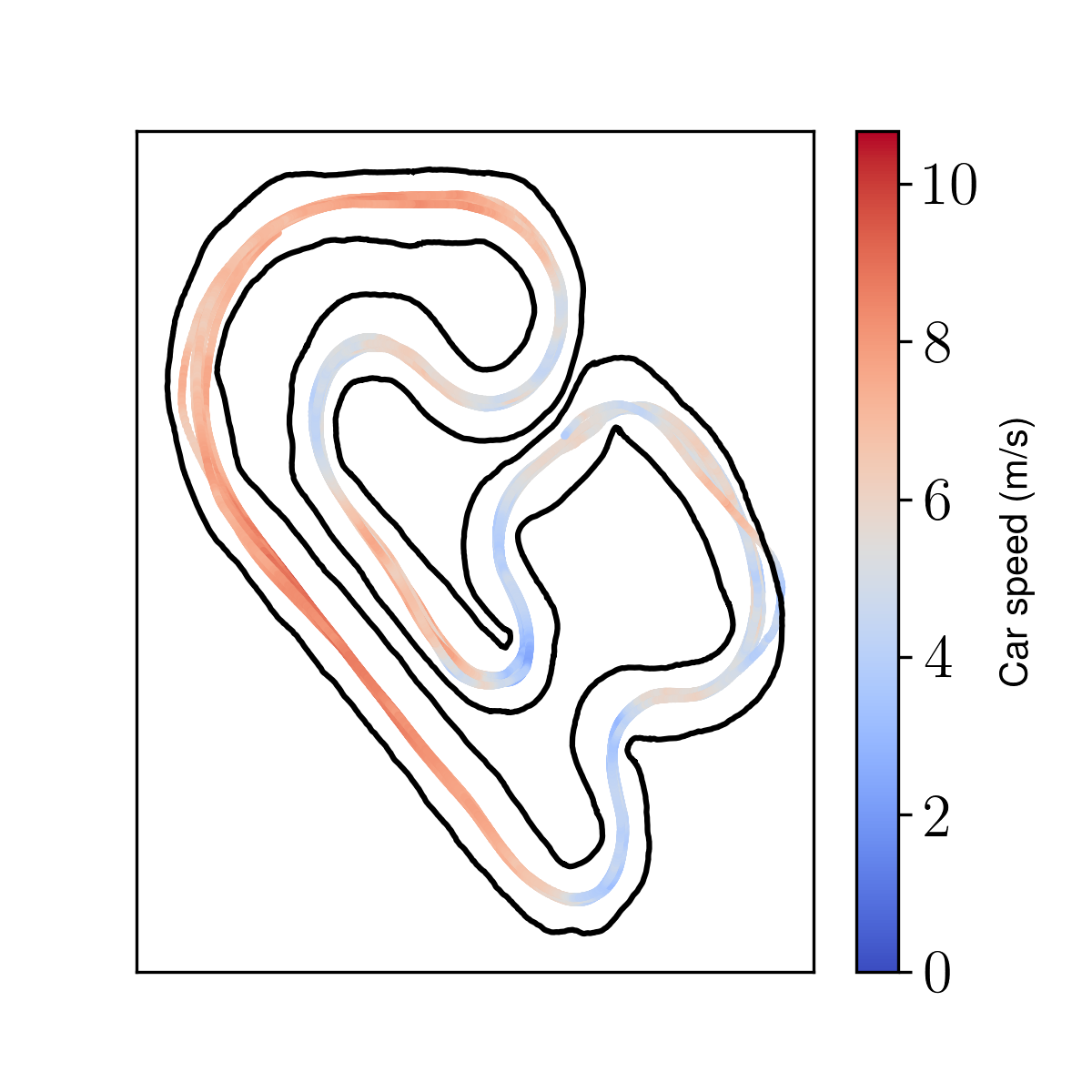}
		\caption{$\gamma_t = 1$}
		\label{fig:1920-1.0}
	\end{subfigure}
	\begin{subfigure}[b]{0.3\textwidth}
		\centering
		\includegraphics[width=\textwidth]{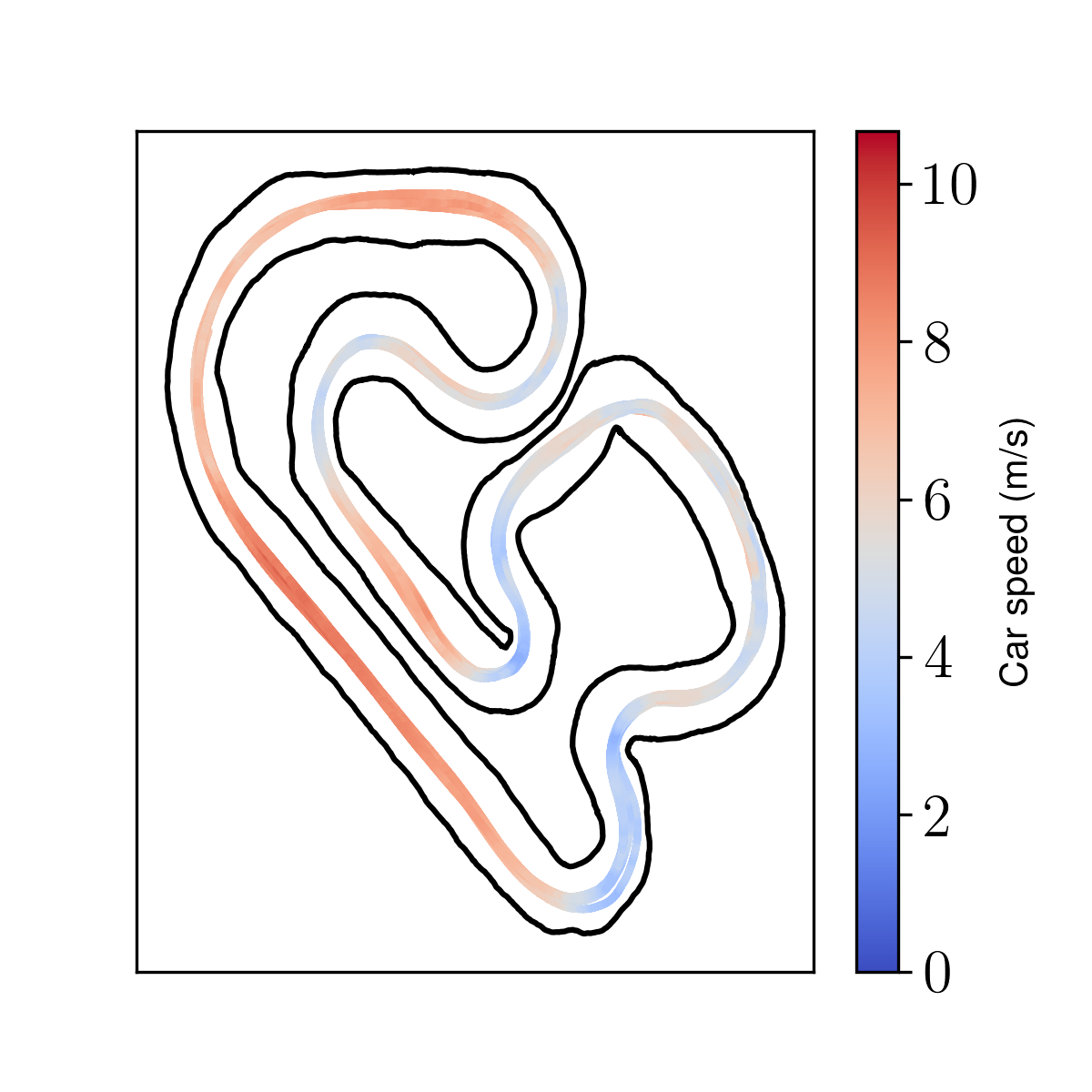}
		\caption{$\gamma_t = 0.8$}
		\label{fig:1920-0.8}
	\end{subfigure}
	\begin{subfigure}[b]{0.3\textwidth}
		\centering
		\includegraphics[width=\textwidth]{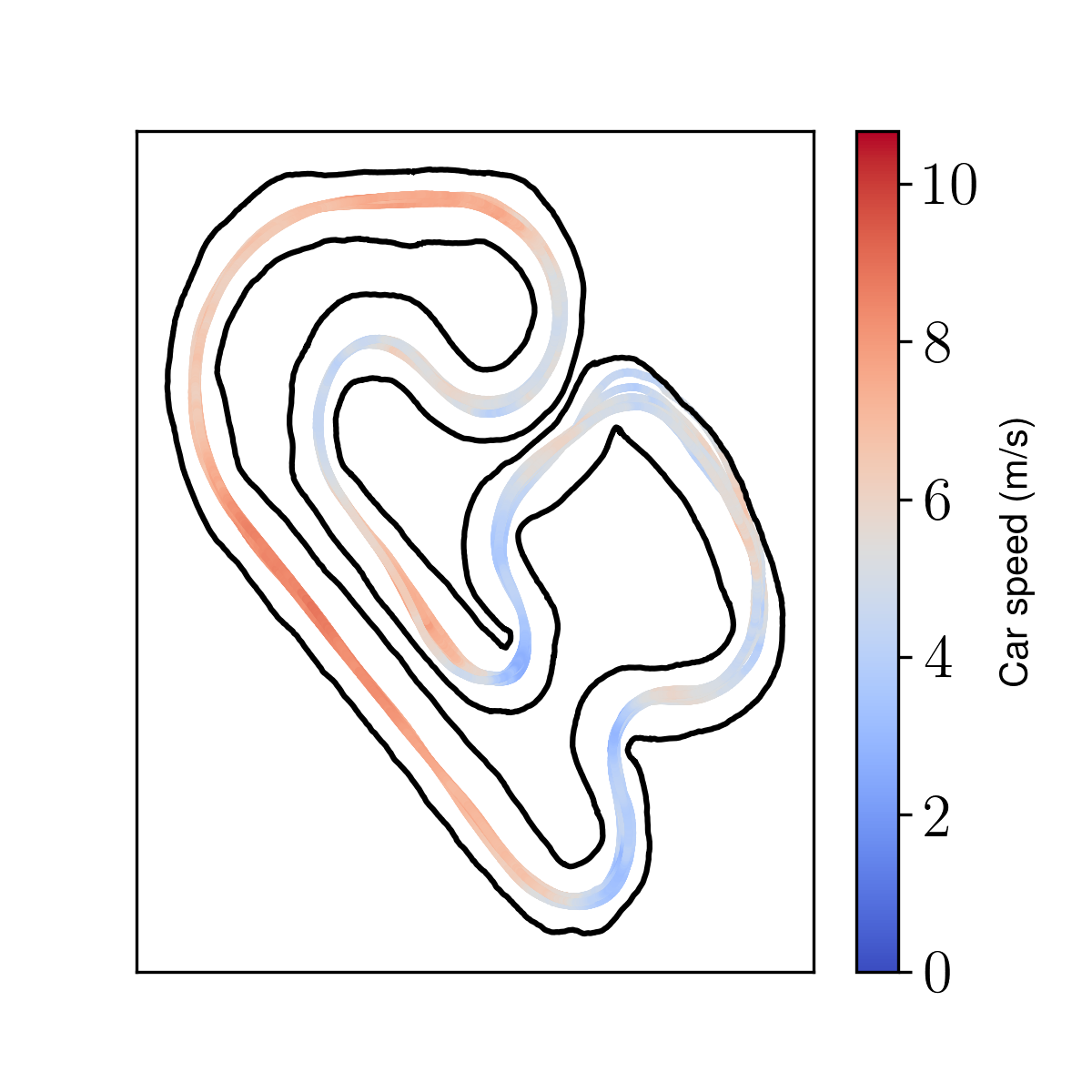}
		\caption{$\gamma_t = 0.6$}
		\label{fig:1920-0.6}
	\end{subfigure}
	\caption{Car speeds with 1920 samples per gradient estimate and target of $9~\mathrm{m/s}$.}
	\label{fig:speed_1920}
\end{figure}
\begin{figure}[h!]
	\centering
	\begin{subfigure}[b]{0.3\textwidth}
		\centering
		\includegraphics[width=\textwidth]{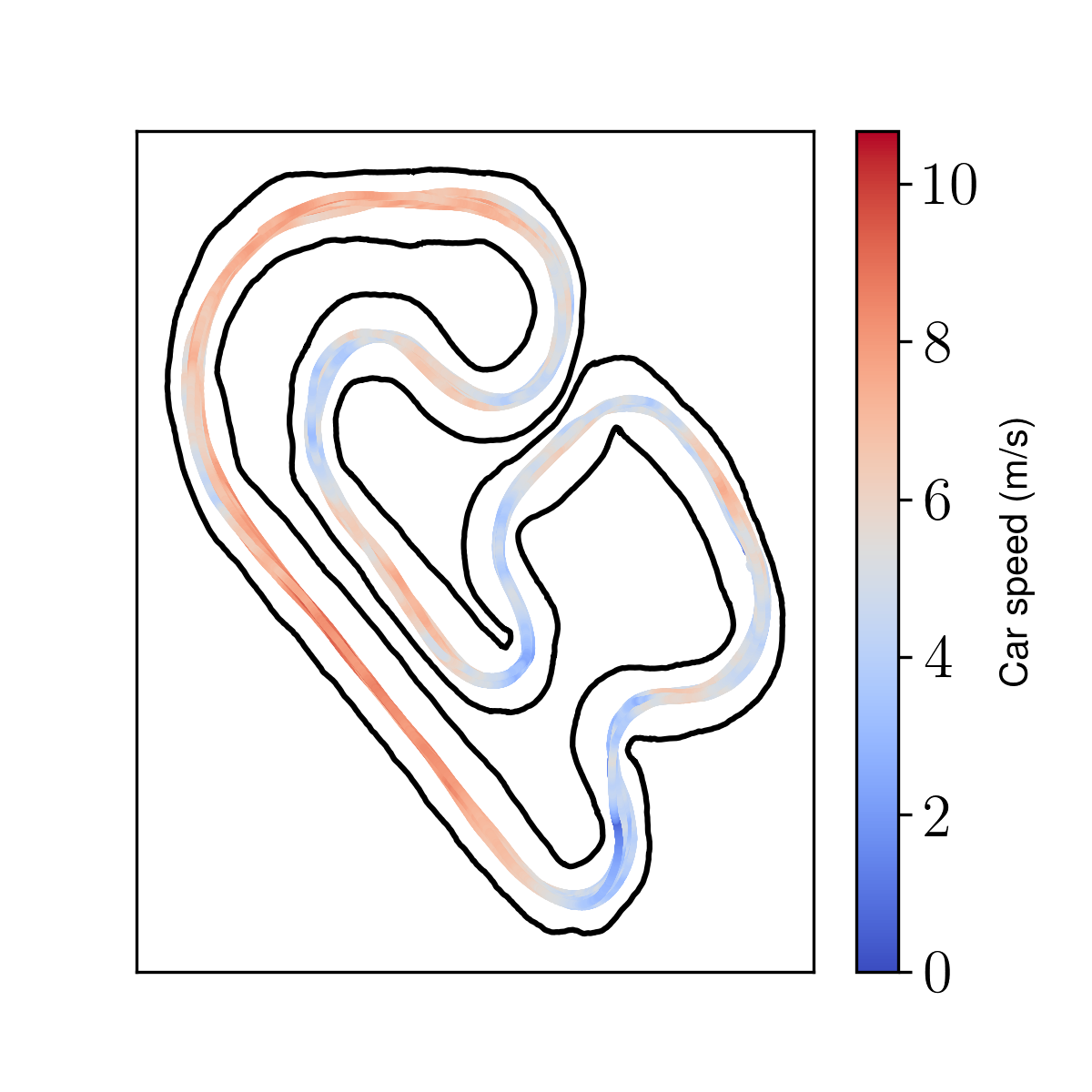}
		\caption{$\gamma_t = 1$}
		\label{fig:64-1.0}
	\end{subfigure}
	\begin{subfigure}[b]{0.3\textwidth}
		\centering
		\includegraphics[width=\textwidth]{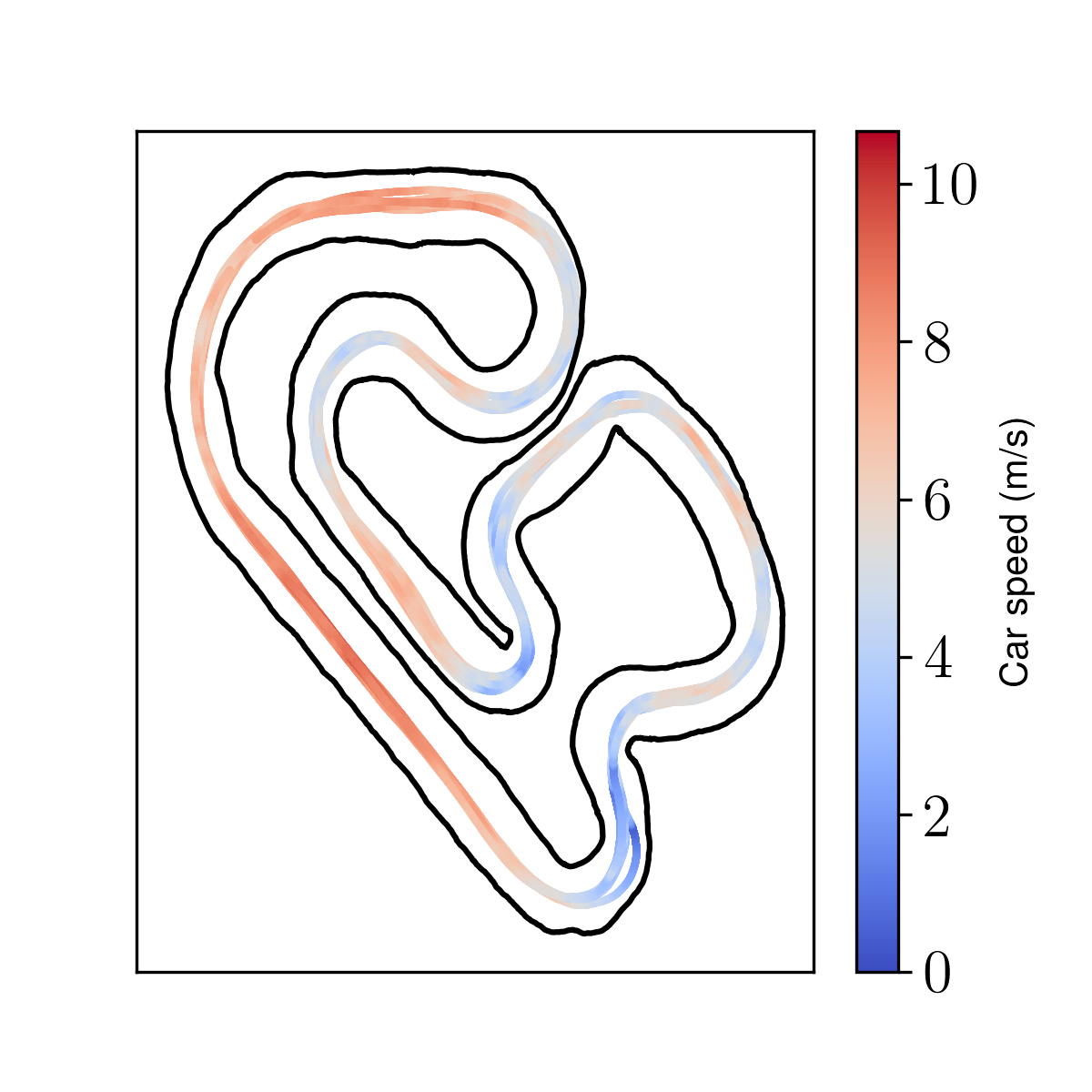}
		\caption{$\gamma_t = 0.8$}
		\label{fig:64-0.8}
	\end{subfigure}
	\begin{subfigure}[b]{0.3\textwidth}
		\centering
		\includegraphics[width=\textwidth]{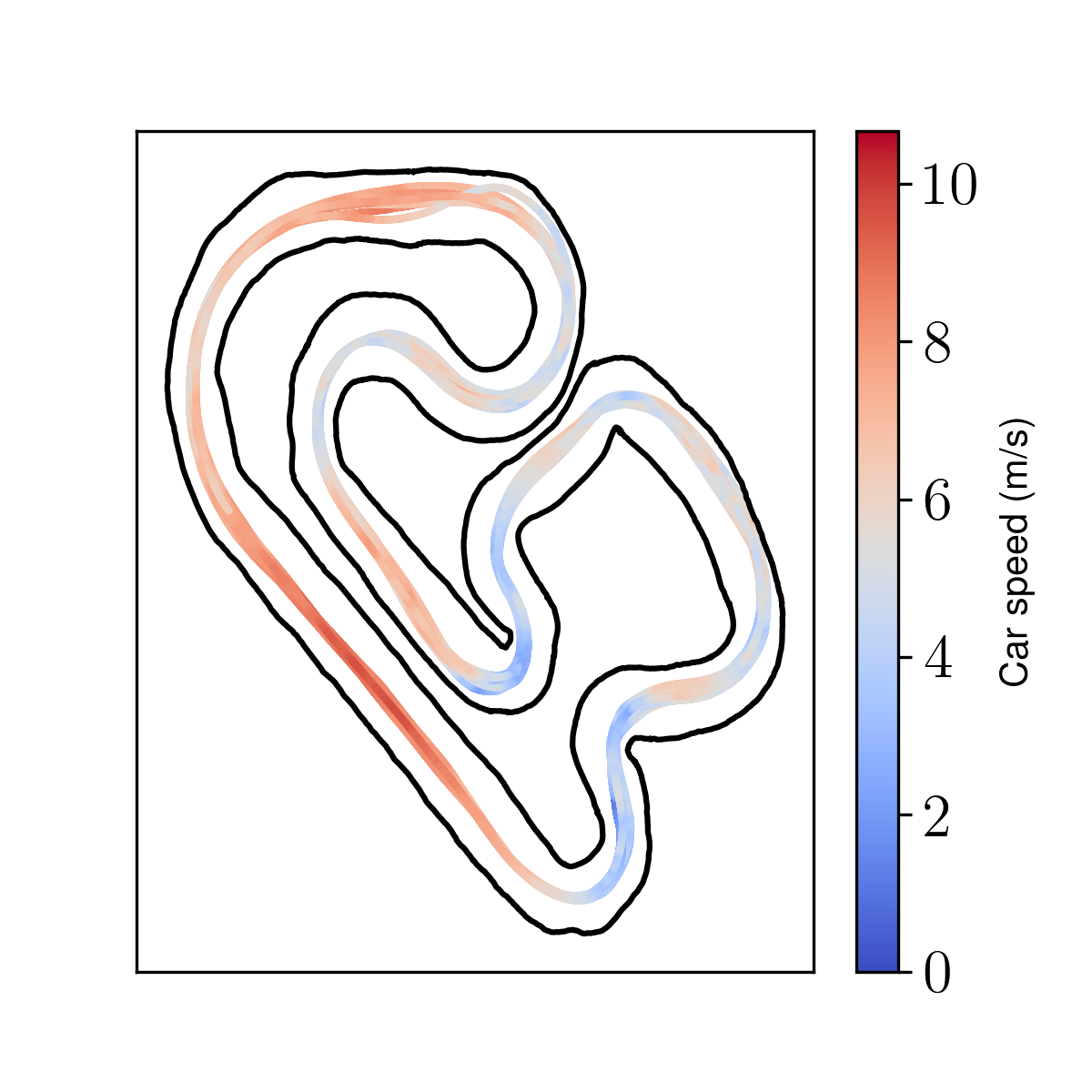}
		\caption{$\gamma_t = 0.6$}
		\label{fig:64-0.6}
	\end{subfigure}
	\caption{Car speeds with 64 samples per gradient estimate and target of $9~\mathrm{m/s}$.}
	\label{fig:speed_64}
\end{figure}
\begin{figure}[h!]
	\centering
	\begin{subfigure}[b]{0.3\textwidth}
		\centering
		\includegraphics[width=\textwidth]{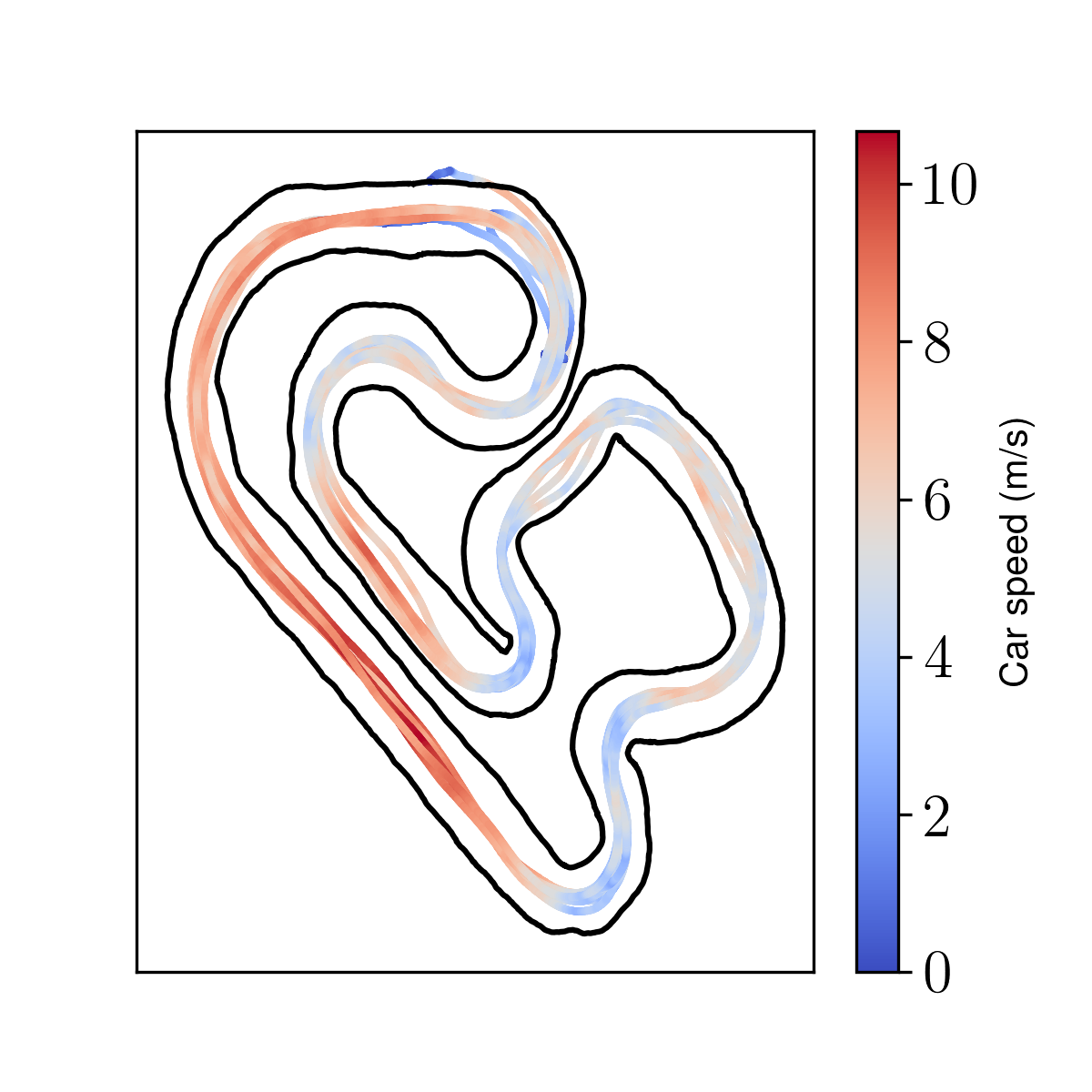}
		\caption{$\gamma_t = 1$}
		\label{fig:fast-64-1.0}
	\end{subfigure}
	\begin{subfigure}[b]{0.3\textwidth}
		\centering
		\includegraphics[width=\textwidth]{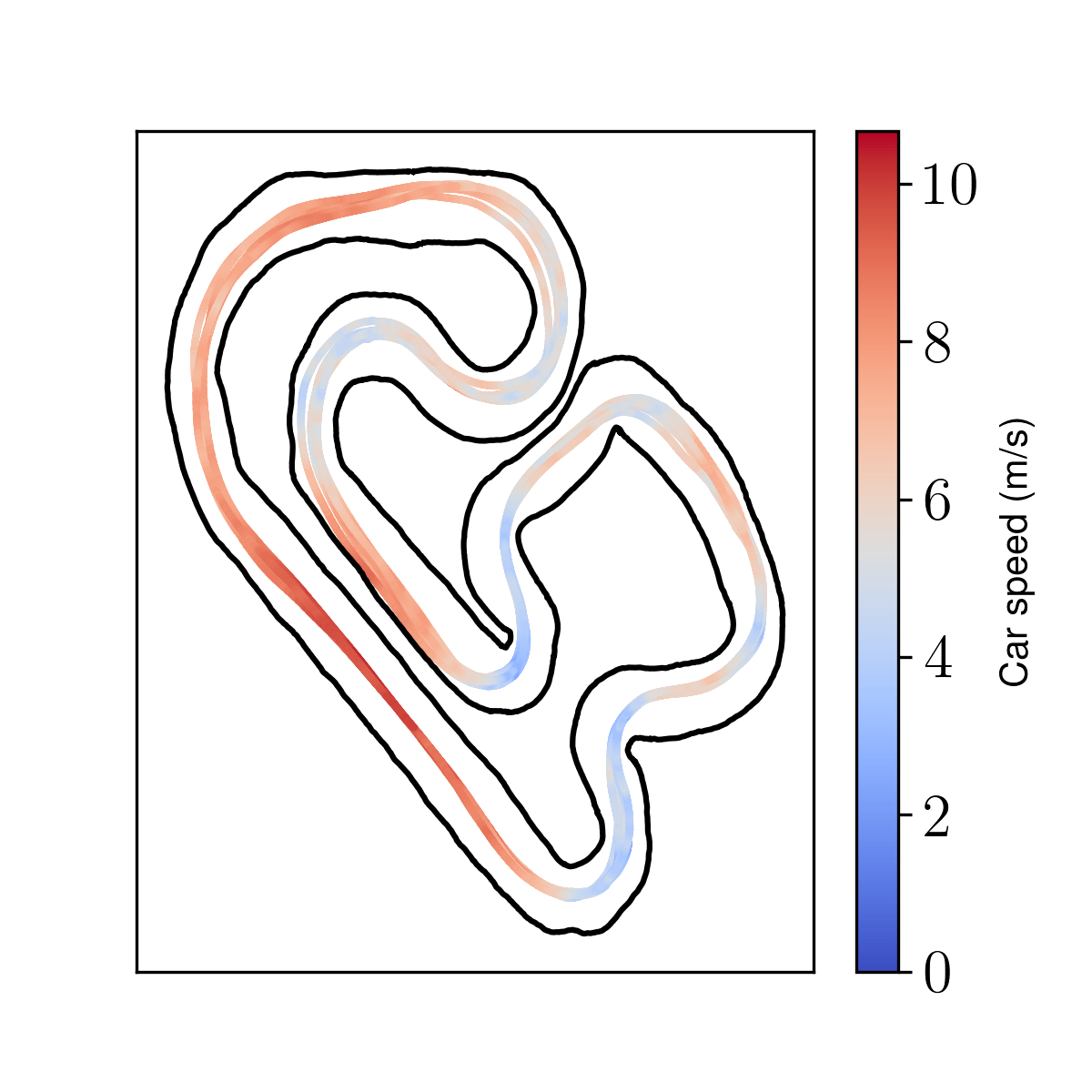}
		\caption{$\gamma_t = 0.6$}
		\label{fig:fast-64-0.6}
	\end{subfigure}
	\caption{Car speeds with 64 samples per gradient estimate and target of $11~\mathrm{m/s}$. In~\cref{fig:fast-64-1.0}, note the crash and U-turn at the top of the plot as well as the wider spread of the paths throughout the whole track. By contrast, in~\cref{fig:fast-64-0.6}, the resulting paths are more consistent, and there are no failure points.}
	\label{fig:fast_speed_64}
\end{figure}

\end{document}